%% file: main.tex
\documentclass{article}

\usepackage{microtype}
\usepackage{graphicx}
\usepackage{subfigure}
\usepackage{wrapfig}
\usepackage{booktabs} % for professional tables

\usepackage{hyperref}

\usepackage[accepted]{icml2024}

\usepackage{ulem}
\usepackage{amsmath}
\usepackage{amssymb}
\usepackage{mathtools}
\usepackage{amsthm}
\usepackage{multirow}

\usepackage[capitalize,noabbrev]{cleveref}

\theoremstyle{plain}
\newtheorem{theorem}{Theorem}[section]

\newtheorem{lemma}[theorem]{Lemma}

\theoremstyle{definition}
\newtheorem{definition}[theorem]{Definition}

\theoremstyle{remark}

\usepackage[textsize=tiny]{todonotes}

\input{math_command}

\begin{document}
\normalem
\twocolumn[
\icmltitle{
% Diffusion Based Prior Knowledge Reuse for Offline to Online Reinforcement Learning\\
% Energy-Guided Diffusion Model (Score-based Generative Modeling with Energy Guidance) for Informative Sampling: Reusing Prior Knowledge in Offline-to-Online Reinforcement Learning\\
% Score-based Generative Modeling with Energy Guidance for Informative Sampling in Offline-to-Online Reinforcement Learning
Energy-Guided Diffusion Sampling for Offline-to-Online Reinforcement Learning
}



\icmlsetsymbol{equal}{*}

\begin{icmlauthorlist}
\icmlauthor{Xu-Hui Liu}{equal,nju}
\icmlauthor{Tian-Shuo Liu}{equal,nju,polixir}
\icmlauthor{Shengyi Jiang}{hku}
\icmlauthor{Ruifeng Chen}{nju,polixir}
\icmlauthor{Zhilong Zhang}{nju,polixir}
\icmlauthor{Xinwei Chen}{polixir}
\icmlauthor{Yang Yu}{nju,polixir}

\end{icmlauthorlist}

\icmlaffiliation{nju}{National Key Laboratory for Novel Software Technology, Nanjing University, China
\& School of Artificial Intelligence, Nanjing University, China}

\icmlaffiliation{hku}{Department of Computer Science, The University of Hong Kong, Hong Kong, China}

\icmlaffiliation{polixir}{Polixir Technologies}


\icmlcorrespondingauthor{Yang Yu}{yuy@nju.edu.cn}

\icmlkeywords{Reinforcement Learning, Diffusion Model}

\vskip 0.3in
]



\printAffiliationsAndNotice{\icmlEqualContribution} % otherwise use the standard text.

\begin{abstract}
% In Offline-to-Online reinforcement learning, the agent is first pre-trained with offline data and then continues to be fine-tuned through online interactions, strategically leveraging informative offline samples to boost online training efficiency. Therefore, how can we extract useful knowledge from offline data? 
% Model-free methods often directly merge offline data into the replay buffer and, when facing distribution shifts, constrain the policy to this fixed dataset, which can compromise overall performance. Vanilla model-based methods, despite fine-tuning model with online data, face challenges from multi-step rollout errors during policy interaction.
% In this work, we propose an energy-guided diffusion sampling approach, termed EGDM. On top of a diffusion model adept at learning hybrid, multi-modal data distributions, our approach additionally employs energy functions, derived from contrastive learning, to guide the sampling process. EGDM is capable of generating desirable on-policy samples for online fine-tuning, while simultaneously maintaining the knowledge of transition fidelity, thereby avoiding the compounding of multi-step model errors.
% We demonstrate the effectiveness of our energy guidance through ablation study. Furthermore, outstanding performance on the D4RL benchmark reveal that our energy-guided diffusion sampling method notably enhance the orginal algorithms.

Combining offline and online reinforcement learning (RL) techniques is indeed crucial for achieving efficient and safe learning where data acquisition is expensive.
% Offline-to-online reinforcement learning, involving pre-training with offline data followed by online fine-tuning, presents a promising approach for decision making.
% Reusing offline data to augment the learning trajectories is vital for efficient online fine-tuning. 
% With properly designed algorithms, the incorporation of offline data into the online phase can straightforwardly and effectively improve the learning process for agents. 
% \sout{Existing works either use policy and Q function derived in the offline phase as initialization or replay offline data directly in the online phase. Both fail to efficiently exploit the dataset's wealth of information.} 
% \zzl{Existing methods replay offline data directly in the online phase, resulting in a significant challenge of data distribution shift and subsequently causing inefficiency in online finetuning.}
Existing methods replay offline data directly in the online phase, resulting in a significant challenge of data distribution shift and subsequently causing inefficiency in online fine-tuning.
% This paper considers leveraging remarkable advancements in generative methods to better utilize the prior information in the offline dataset. 
% Offline-to-online reinforcement learning strikes a balance between online interaction costs and the suboptimality inherent in offline learning. While existing works address performance degradation due to the convention of two learning phases, they fall short in fully leveraging the prior knowledge embedded in the offline dataset. Typically, they initialize the policy and Q function derived in the offline phase but lack mechanisms to exploit the dataset's wealth of information.
% \sout{Based on the generative model which achieves remarkable advancements recently,} \zzl{To address this issue,}
To address this issue, we introduce an innovative approach, \textbf{E}nergy-guided \textbf{DI}ffusion \textbf{S}ampling (EDIS), which utilizes a diffusion model to extract prior knowledge from the offline dataset and employs energy functions to distill this knowledge for enhanced data generation in the online phase. 
% The generated samples conform online fine-tuning distribution without oblivion of transition fidelity.
% , outperforming model-based and resampling techniques. 
The theoretical analysis demonstrates that EDIS exhibits reduced suboptimality compared to solely utilizing online data or directly reusing offline data. EDIS is a plug-in approach and can be combined with existing methods in offline-to-online RL setting. By implementing EDIS to 
off-the-shelf methods Cal-QL and IQL, we observe a notable 20\% average improvement in empirical performance on MuJoCo, AntMaze, and Adroit environments. Code is available at \url{https://github.com/liuxhym/EDIS}.
% Offline to online reinforcement learning achieves a balance between the cost of online interaction and the suboptimality inherent in offline learning. Existing works focus on the performance degrdation caused by the convention of two learning phase. However, they failed to sufficiently utilize prior knowledge contained in the offline dataset by using the policy and Q function derived in the offline phase as the initialization. In this work, we propose an energy-guided diffusion sampling approach, termed EGDS, which uses diffusion model to extract the prior knowledge of the offline dataset, and energy functions to distill the knowledge to augment data generation of online phase. EGDS is capable of generating desirable on-policy samples for online fine-tuning, while simultaneously maintaining the knowledge of transition fidelity. The theoretical analysis verifies the intuition that EGDS has better utilization of prior knowledge while avoiding compounding error of data generation. Empirically, EGDS.....

\end{abstract}

\section{Introduction}
Reinforcement learning (RL)~\cite{sutton} has demonstrated remarkable efficacy in diverse decision-making tasks, such as sequential recommendation systems~\cite{rs1, rs2}, robotic
locomotion skill learning~\cite{rc1, sac} and operations research~\cite{hubbs2020or,wang2023learning,wang2024learning,Ling_Wang_Wang_2024}. Notably, RL methods, including online~\cite{dqn, liu2021regret,foda} and offline~\cite{bcq, CQL, hybrid,ada} ones, have been employed to achieve superior performance. However, online RL methods often demand extensive data collection through interacting with the environment, a process that can be time-consuming or risky. In contrast, offline RL methods utilize pre-existing datasets to train a new policy, avoiding the need for resource-intensive online interactions but often suffering from suboptimality due to limited data.

To mitigate the challenges incurred due to data limitations in online and offline RL, researchers have proposed the offline-to-online setting, aiming to overcome the cost of online interaction and the suboptimality inherent in offline learning. In this setting, the agent undergoes two learning phases: first, it learns from an offline dataset, and then it fine-tunes through a limited number of online interactions~\cite{balanced,iql,calql}. After the offline phase, value functions or policies are derived and used to initialize the online phase. Therefore, the algorithm can utilize prior knowledge in the offline dataset to help reduce the cost of the online phase. However, a significant drawback in this paradigm arises from the incomplete utilization of the offline dataset, as the information extracted is limited to the pre-trained policy.

While previous works have effectively addressed the performance drop because of the objective mismatch of the two-stage learning in offline-to-online settings~\cite{calql, balanced,awac}, the incomplete utilization of the prior knowledge in the offline dataset remains an unexplored challenge. Some studies directly replayed samples in the online phase for data augmentation~\cite{balanced}, leading to performance improvements. Nevertheless, this approach neglects the distribution shift issue, as the data distribution in the offline dataset may differ from that induced by the current policy. Distribution shift has shown adverse effects even in off-policy reinforcement learning~\cite{lfiw, liu2021regret}, which incorporates a replay buffer, let alone when utilizing an offline dataset. Consequently, a fundamental question arises:

\textit{Is it feasible to generate samples without distribution shift based on an offline dataset?}
% However, previous works formulated the two learning phase in a separated way: in the offline pre training phase, they obtain a pertained policy or Q function; In the online phase,  the trained networks served as an initialization. In other words, these methods use the offline dataset to do policy evaluation and policy improvement, which causes information loss of the prior knowledge about the environment provided by the offline dataset. This is because the policy evaluation and improvement are highly related to the learned policy, while the offline dataset contains some information has no relationship with it.

Model-based methods~\cite{Wang_Wang_Zhou_Li_Li_2022,moto,xionghui} generate such samples by first learning a model and using it for rollout. However, this style of generation inevitably suffers from the issue of compounding errors: the error of the learned model itself will accumulate during the multi-step rollout. Therefore, we aim to generate samples with the desired distribution directly. Motivated by the recent achievements of the diffusion model in image generation~\cite{diff, elucidate}, we want to explore its potential application in RL sample generation.
% for generating  desired samples that follow a certain distribution directly in RL. 
However, utilizing a diffusion model trained on an offline dataset introduces a challenge—it can only generate samples adhering to the dataset distribution, thus still being susceptible to distribution shift issues. 

% Drawing inspiration from the work of \citep{contrastive}, who introduced the concept of using an energy function to guide the diffusion sampling process and achieve the desired distribution, 
The desired distribution for RL has three crucial characteristics: i) the state distribution should align with that in the online training phase, ii) actions should be consistent with the current policy, and iii) the next states should conform to the transition function. To achieve this, we formulate three distinct energy functions to guide the diffusion sampling process, ensuring alignment with the aforementioned features.
% Notably, the diffusion model, having been trained on the offline dataset, possesses prior knowledge. The proposed energy functions distill this knowledge to assist the online learning phase effectively. 
The developed algorithm is named as \textbf{E}nergy-Guided \textbf{DI}ffusion \textbf{S}ampling (EDIS).
% we present a constructive response to the posed question, showcasing the feasibility of generating samples while mitigating distribution shift effectively leveraging offline dataset.
By presenting the new algorithm, we showcase the feasibility of generating additional, useful samples  while mitigating distribution shift, which will effectively help leverage the offline dataset.

% Inspired by the recent success of diffusion model~\cite{diff, elucidate} on image generation, we consider use it to generate desired samples for RL method. However, the diffusion model trained on offline dataset can only generate samples subject to the dataset distribution, which still suffers distribution shift issue. The desired distribution has three main features: i) the state distribution confirms the online training phase, ii) actions be in line with the current action, iii) the next states should be in line with the transition function. Recently, \citet{contrastive} proposed to use energy function to guide the diffusion sampling process to get the desired distribution exactly. Inspired by this idea, we construct three energy functions to perform the guidance to achieve the three features. Notably, since the diffusion model is trained based on the offline dataset, it is equipped with the prior knowledge. While energy functions distill the the knowledge out to assist the online learning phase. With the new algorithm Energy-Guided Diffusion Sampling (EGDS), we provide a positive answer to this question. 

The theoretical analysis shows that EDIS exhibits reduced suboptimality compared to solely utilizing online data or directly replaying offline data. Additionally, it circumvents the compounding error issue commonly associated with model-based methods. When incorporating EDIS into existing methods, Cal-QL~\cite{calql} and IQL~\cite{iql}, we observe 20\% performance improvement on average across MuJoCo, AntMaze, and Adroit environments. Furthermore, our experiments demonstrate that purely model-based methods fail to achieve performance improvements due to the compounding errors in the model during rollouts, leading to an inaccurate state distribution. The ablation study further validates the effectiveness of the energy functions to guide the diffusion models.

% \textcolor{red}{We summarize our contributions as follows:}

% \begin{itemize}
%     \item \textcolor{red}{Our paper introduces the innovative EDIS algorithm, which is designed for seamless integration with existing offline-to-online strategies, resulting in a marked enhancement of their performance.}
%     \item \textcolor{red}{We formulate a novel theoretical framework for analyzing offline-to-online RL, capable of assessing the error margins of prevailing algorithms and illustrating how EDIS successfully mitigates the issue of compounding errors.}
% \end{itemize}
% the introduction of the diffusion model guided by energy functions.
% We provide a motivating example to illustrate the information loss issue intuitively. [theoretical results]. To overcome this issue, we propose a new method, called [name]. [name] utilize the powerful diffusion model to extract prior knowledge from the offline dataset and then promote the online data generation process. With the data synthesized by the diffusion model, the online phase can achieve better performance with small online interaction budget. We test our algorithm in MuJoCo benchmark….
\vspace{-0.5em}
\section{Preliminaries}
\subsection{Markov Decision Process and RL}
Let $M=(\gS, \gA, P, r, \gamma, \rho_0)$ be an MDP, where $\gS$ is the state space, $\gA$ is the action space, $P:\gS\times \gA\to \Delta(\gS)$ is the transition function ($\Delta(\cdot)$ is the probability simplex), $r:\gS\times \gA\to [0, R_{\max} ]$ is the reward function, $\gamma\in [0,1)$ is the discount faction and $\rho_0$ is the initial distribution over states. A policy $\pi:\gS\to \Delta(\gA)$ describes a distribution over actions for each state. The goal of RL is to learn the best policy $\pi^*$ that maximizes cumulative discounted reward, i.e., $\sum_t\E_{a_t\sim \pi^*}\gamma^tr(s_t,a_t)$. The value function and Q function of policy $\pi$ are
    $V^\pi(s)=\sum_t\E_{a_t\sim \pi(s_t)}[\gamma^tr(s_t,a_t)|s_0=s]$,
    $Q^\pi(s,a)=\sum_t\E_{a_t\sim \pi(s_t)}[\gamma^tr(s_t,a_t)|s_0=s, a_0=a]$.
$V^*$ and $Q^*$ be the shorthand for $V^{\pi^*}$ and $Q^{\pi^*}$ respectively. To facilitate later analysis, we introduce the \textit{discounted stationary state distribution} $d^\pi(s)=\sum_t\gamma^t\textnormal{Pr}(s_t=s;\pi)$ and the \textit{discounted stationary state-action distribution} $d^\pi(s,a)=\sum_t\gamma^t\textnormal{Pr}(s_t=s, a_t=a;\pi)$. There are, in general, two learning paradigms of RL: online RL, where the agent can learn from interacting with the environment; and offline RL, where the agent can only learn from a given dataset $\gD=\{(s,a,r,s')\}$, possibly collected by another policy.

% \subsection{Offline-to-online Reinforcement Learning}
% Given access to an offline dataset $\gD=\{(s,a,r,s')\}$, offline-to-online RL aims to train a good policy utilizing both the offline dataset and the online interaction in MDP $M$. Our goal during fine-tuning is to obtain the optimal policy with the smallest number of online samples.

\subsection{Generative Modeling by Diffusion Model}
Diffusion models~\cite{diff, elucidate} are a class of generative models inspired by non-equilibrium thermodynamics. Given a dataset $\{\boldsymbol{x}_0^{(i)}\}_{i=1}^N$ with $N$ samples of $D$-dimensional random variable $\boldsymbol{x}$ from an unknown data distribution
$p_0(\boldsymbol{x}_0)$.
Diffusion models gradually add Gaussian noise at time 0 to $T$ according to noise levels $\sigma_{\max}=\sigma_T>\sigma_{T-1}>\cdots>\sigma_0=0$ so that at each noise level, $\boldsymbol{x}_t^{(i)}\sim p(\boldsymbol{x}_t;\sigma_t)$.
The distribution of the endpoint
$\boldsymbol{x}_T^{(i)}\sim\gN(0,\sigma_{\max}^2 \boldsymbol{I})$ 
is indistinguishable from pure Gaussian noise, while the starting point 
$\boldsymbol{x}^{(i)}_0$ aligns with the distribution of the original dataset.
% \begin{equation}\label{eq_add_noise}q(\boldsymbol{x}_{t}|\boldsymbol{x}_{t-1})=\mathcal{N}(\boldsymbol{x}_t;\sqrt{\alpha_t}\boldsymbol{x}_{t-1},(1 - \alpha_t) \boldsymbol{I}).
% \end{equation}
Then the reverse process 
%$D_\theta(\boldsymbol{x}_{t-1}|\boldsymbol{x}_t)=\mathcal{N}(\boldsymbol{x}_{t-1}|\mu_\theta(\boldsymbol{x}_t,t),\Sigma_t)$ 
starts from a Gaussian distribution and iteratively denoises samples using the trained model, ultimately recovering the target distribution. The process can be interpreted as \emph{SDE}s:
\begin{align}\label{eq_sde}
    \mathrm{d}\boldsymbol{x}_{\pm} = -\dot \sigma(t)\sigma(t)\nabla_{\boldsymbol{x}}&\log p(\boldsymbol{x};\sigma(t))\mathrm{d}t \notag \\
    \pm \beta(t)\sigma^2(t)\nabla_{\boldsymbol{x}}&\log p(\boldsymbol{x};\sigma(t))dt+\sqrt{2\beta(t)}\sigma(t)\mathrm{d}\omega_t,
\end{align}
% \begin{equation}
% \mathrm{d}\boldsymbol{x} = \left[\textup{f}(\boldsymbol{x}, t) - g(t)^2\nabla_{\boldsymbol{x}}\log q_t(\boldsymbol{x})\right] + g(t)\mathrm{d}\bar{\textup{w}},
% \end{equation}
where $\omega_t$ is a standard Wiener process, $d\boldsymbol{x}_+$ and $d\boldsymbol{x}_-$ are separate SDEs for moving forward and backward in time. 
% $\textup{f}(\cdot, t)$ is the \emph{drift} coefficient of $\boldsymbol{x}(t)$, $g(\cdot)$ represents the scalar function known as the \emph{diffusion} coeffcient of $\boldsymbol{x}(t)$ and 
The only unknown term is the \emph{score function} $\nabla_{\boldsymbol{x}}\log p_t(\cdot)$.
\citet{elucidate} considers training a denoiser $D_\theta(\boldsymbol{x}_t;\sigma)$ on an L2 denoising minimization objective:
\begin{equation}
    L_{\textnormal{VLB}}(\theta):=\E_{\boldsymbol{x}\sim p_0(\boldsymbol{x}_0), \epsilon\sim \gN(0,\sigma^2I)}\left\|D_\theta(\boldsymbol{x}+\epsilon;\sigma)-\boldsymbol{x}\right\|^2_2,
    \label{diffusionloss}  
\end{equation}

where $\sigma$ is the standard deviation of Gaussian noise. After training a diffusion model for estimating the score function with 
$\nabla_{\boldsymbol{x}}\log p_0(\boldsymbol{x};\sigma) = (D_\theta(\boldsymbol{x};\sigma) - \boldsymbol{x})/\sigma^2$
, we can fastly generate samples by solving the backward \emph{SDE}.

\section{EDIS: Energy-Guided Diffusion Sampling}
To extract prior knowledge from the offline dataset and generate samples to conform to the online data distribution, we introduce our innovative approach, named \textbf{E}nergy-guided \textbf{DI}ffusion \textbf{S}ampling (EDIS). 
% The core objective of EDIS is to effectively leverage prior knowledge in offline dataset to upsample online experiences. 
At the heart of our method is to accurately generate a desired online data distribution, denoted as $q_{\pi}(s, a, s')$, from pre-gathered data. The distribution does not include reward $r$ because we assume that the reward function $r(s,a)$ is accessible, either directly or through learning from the dataset. To achieve this, we have integrated a diffusion model into our framework, capitalizing on its exceptional capability for modeling complex distributions.

\subsection{Distribution Adjustment via Energy Guidance}
One challenge in this process is the inherent limitation of directly training a diffusion model on an offline dataset. Such a model typically yields an offline data distribution $p_\gD(s,a,s')$, which does not align perfectly with online data and causes distribution shift issues. To address this, our method needs to guide the diffusion sampling process towards the online distribution. This is achieved by decomposing the online data distribution into the following form:
\begin{equation}\label{eq_modified_dist}
q_{\pi}(s,a,s') \propto p_\theta(s,a,s')e^{-\mathcal{E}(s,a,s')},
\end{equation}
where $p_\theta(s,a,s')$ is the distribution generated by the denoiser network, parameterized by $\theta$. $\mathcal{E}(s,a,s')$ is the energy function, which serves as the guidance to bridge the gap between generated distribution and online data distribution. The following theorem shows such an energy function exists.

\begin{theorem}\label{thm_energy}
    Let $p_\theta(s)$ be the marginal distribution of $p_\theta(s,a,s')$, $p_\theta(a|s)$ and $p_\theta(s'|s,a)$ be the conditional distribution of $p_\theta(s,a,s')$ given $s$ and $(s,a)$. Eq.~(\ref{eq_modified_dist}) is valid if the energy function $\mathcal{E}(s,a,s')$ is structured as follows:
    \begin{equation}
    \gE(s,a,s') = \gE_{1}(s) + \gE_{2}(a|s)+\gE_{3}(s'|s,a),
    \end{equation}
    such that $e^{\gE_1(s)}\propto \frac{p_\theta(s)}{d^\pi(s)}$, $e^{\gE_2(a|s)}\propto\frac{p_\theta(a|s)}{\pi(a|s)}$, $e^{\gE_3(s'|s,a)}\propto\frac{p_\theta(s'|s,a)}{T(s'|s,a)}$.
\end{theorem}

This theorem indicates that the energy function can be decomposed into three distinct parts. Each part is responsible for aligning the generated distribution with different aspects of the online data: the online state distribution, the current policy action distribution, and the environmental dynamics. 

\subsection{Learning Energy Guidance by Contrastive Energy Prediction}
\label{para:designenergy}
First, we concentrate on the first component, $\gE_{1}(s)$. We assume that the energy is estimated using a neural network denoted as $\gE_{\phi_1}(s)$. Let $K$ and $K_{\textrm{neg}}$ be two positive numbers. Given $s_1, s_2, \dots, s_K$, $K$ i.i.d. samples drawn from the distribution $p_\theta(s)$, and $s_i^1, s_i^2, \dots, s_i^{K_{\textrm{neg}}}$, $K_{\textrm{Neg}}$ negative samples for $s_i$. We employ the Information Noise Contrastive Estimation (InfoNCE) loss~\cite{infonce}:
\begin{equation}\label{eq_energy_1}
\gL(\phi_1)=-\sum_{i=1}^K\log\frac{e^{-\gE_{\phi_1}(s_i)}}{e^{-\gE_{\phi_1}(s_i)}+\sum_{j=1}^{K_{\textrm{neg}}}e^{-\gE_{\phi_1}(s_i^j)}},
\end{equation}

% where the probability term is given by:
% $$p_{\phi_1}(s_i|\{\hat{s}_i^j\}_{j=1}^{K_{\textrm{neg}}})=.$$

Then, we devise positive and negative samples to achieve the target energy function established by Thm.~\ref{thm_energy}. Suppose the distribution of positive samples is $\mu(s)$, the distribution of negative samples is $\nu(s)$, the final optimized results is
$e^{\gE_{\phi_1}(s)} \propto \frac{\nu(s)}{\mu(s)}$~\cite{infonce}.
Compared to the function indicated by Thm.~\ref{thm_energy}, the result can be achieved by selecting $\mu(s)=d^\pi(s)$, $\nu(s)=p_\theta(s)$. Following the approach of \citet{lfiw, liu2021regret}, we construct a \textit{positive buffer}, containing only a small set of trajectories from very recent policies. The data distribution in this buffer can be viewed as an approximation of the on-policy distribution $d^\pi(s)$. While $p_\theta(s)$ is the distribution of the data generated during the denoising steps. Therefore, the positive samples is sampled from the \textit{positive buffer} and the negative samples is sampled from the denoiser.

For the remaining terms $\gE_2(a|s)$ and $\gE_3(s'|s,a)$, we also parameterize them as $\gE_{\phi_2}(a|s)$ and $\gE_{\phi_3}(s'|s,a)$, and employ infoNCE loss to estimate them. Similarly, the loss computation requires the generation of positive and negative samples.
% We first derive the distribution of the positive samples for $\gE_2(a)$ and $\gE_3(s')$ are $\pi(a|s)$ and $T(s'|s,a)$. 
 Differently from $\gE_{\phi_1}(s)$, here the samples from the offline buffer can also used for positive-sample generation.
% To better utilize the data, we use not only the online data but also the offline data to train the two energy functions. 
% Suppose the online data distribution is $p_{\gB}(s,a,s')$, then the joint data distribution in the buffer is $$p(s,a,s')=\frac{N}{N+n} p_\gD(s,a,s')+\frac{n}{N+n}p_\gB(s,a,s'),$$ 
% where $n$ and $N$ are the number of samples in online dataset and offline dataset respectively. 
 For $\gE_{\phi_2}(a|s)$, the positive samples are obtained by first sampling $s$ from the offline and online distribution and outputting the action of the current policy on $s$. For $\gE_{\phi_3}(s'|s,a)$, positive samples are obtained by sampling $(s, a, s')$  from the offline and online distribution and outputting $s'$ directly.
% Conditioned on the sampled $s$, the positive samples of $\gE_{\phi_2}(a)$ are generated by the current policies. While for the positive samples of $\gE_{\phi_3}(s')$, we directly sample $s'$ of the sampled $(s,a,s')$.
Negative samples for both $\gE_{\phi_2}(a|s)$ and $\gE_{\phi_3}(s'|s,a)$ are generated using the denoiser. For each positive samples of $\gE_{\phi_2}(a|s)$, the negative samples are generated conditioned on the given states. While for each positive samples of $\gE_{\phi_3}(s'|s,a)$, they are generated conditioned on the sampled states-action pairs. This method of conditional sample generation is inspired by techniques used in decision-making scenarios~\cite{DecisionDiffuser, diffuser}. 
% To obtain the positive samples of $\gE_{\phi_3}(s')$, we directly sample $s'$ of $(s,a,s')$ from the joint distribution of offline and online data. Negative samples are generated with denoiser based on sampled states-action pairs. The conditional generation techniques also mirrored from~\cite{DecisionDiffuser, diffuser}.
% that manage given initial states or inpainting scenarios. 
The loss functions of parameterized $\gE_{\phi_2}(a)$ and $\gE_{\phi_3}(s')$ are:
\begin{align}
&\gL(\phi_2)=
% \notag \\
% &\quad 
-\sum_{k=1}^{K_2}\sum_{i=1}^{K_1}\log\frac{e^{-\gE_{\phi_2}(a_{ik}|s_i)}}{e^{-\gE_{\phi_2}(a_{ik}|s_i)}+\sum_{j=1}^{K_{\textrm{neg}}}e^{-\gE_{\phi_2}(a_{ik}^j|s_i)}}, \label{eq_energy_2} \\ 
&\gL(\phi_3)=-\sum_{i=1}^K\log\frac{e^{-\gE_{\phi_3}(s'_i|s_i, a_i)}}{e^{-\gE_{\phi_3}(s'_i|s_i,a_i)}+\sum_{j=1}^{K_{\textrm{neg}}}e^{-\gE_{\phi_3}(s_i^{j'}|s_i,a_i)}}, \label{eq_energy_3}
\end{align}
where $K_1$ is the number of states sampled, $K_2$ is the number of actions sampled for each state.

\subsection{Sampling under Energy Guidance}
To realize this distribution with diffusion models, we need to calculate its score function in the sampling process, taking into account the energy function designed 
 in Sec.~\ref{para:designenergy}:
\begin{equation}
\begin{aligned}
    &\nabla_{(s,a,s')} \log q_\pi(s,a,s') \\
    &= \nabla_{(s,a,s')} \log p_\theta(s,a,s')
    - \nabla_{(s,a,s')}\gE(s,a,s')  
\end{aligned}
\end{equation}

In the denoising process, we need to obtain the score function at each timestep. Denote the forward distribution at time $t$ starting from $p_{0}(s, a, s')$ as $p_t(s,a,s')$. 
Remember that the denoiser model $D_{\theta}(s, a, s'; \sigma)$ is designed to match the score with the expression:
\begin{equation}
    \nabla\log p_\theta(s,a,s') = (D_\theta(s,a,s';\sigma) - (s,a,s'))/\sigma^2.
\end{equation} 

Thus, we can obtain the gradient through the denoiser model. Then, the key problem is to obtain the intermediate energy guidance at time $t$, which is addressed in the following theorem.

\begin{theorem}[Thm.~3.1 in \cite{contrastive}]
    Denote $q_t(x_t) := \int q_{t}(x_t|x_0)q_0(x_0)dx_0$ and $p_t(x_t) := \int p_{t}(x_t|x_0)p_0(x_0)dx_0$ as the marginal distributions at time t, and define
    $$\gE_t(x_t):=\left\{
    \begin{aligned}
    &\gE(x_0), &t=0,\\
    &-\log\E_{p_{t}(x_0|x_t)}[e^{-\gE(x_0)}], &t>0.
    \end{aligned}\right.
    $$
    Then the score functions satisfy
    $$\nabla_{x_t}\log q_t(x_t)=\nabla_{x_t}\log p_t(x_t)-\nabla_{x_t}\gE_t(x_t).$$
\end{theorem}
\vspace{-1mm}
% The intermediate energy guidance can be denoted as 
% \begin{align}
% \nabla_{(s,a,s')_t}\log q_{\pi}(s,a,s')_t &= \nabla_{(s,a,s')_t}\log p_t(s,a,s')_t \notag \\ &- \nabla_{(s,a,s')_t}\gE(s,a,s')_t,
% \end{align}
% The former part represents the Langevin sampling process, while the latter is shaped by the energy functions training with contrastive loss.
% Note that intermediate samples are designated as negative samples, and the intermediate energy functions remain consistent with the definitions in \eqref{energyfunction}.
% Let $x_0$ to be the original positive samples. According to this theorem, we need to model the energy function at timestep $t$. Similar to Sec.~\ref{para:designenergy}, we use contrastive learning method. The positive samples for the energy function at time $t$ are the original samples with $t$ steps of adding noise, i.e.,  $x_t := \alpha_t x_0 + \sigma_t \epsilon$, where $\alpha_t$ and $\sigma_t$ are defined in Eq.~(\ref{eq_add_noise}) and (\ref{diffusionloss}), and $\epsilon$ represents Gaussian noise. The negative samples are generated by the denoiser at time $t$.

Let $x_0$ denote the set of original positive samples. According to this theorem, we aim to formulate the energy function at timestep $t$. Drawing inspiration from Sec.~\ref{para:designenergy}, we opt for a contrastive learning approach.
Within this method, these positive samples for the energy function at $t$ timestep are derived from the original samples subjected to $t$ steps of noise addition, i.e., $x_t\sim p(\boldsymbol{x}_i;\sigma_t)$, where $\sigma_t$ is chosen as~\cite{elucidate}:
$$\sigma_t=\left(\sigma_{\max}^{\frac{1}{\rho}}+\frac{t}{T-1}(\sigma_{\min}^{\frac{1}{\rho}}-\sigma_{\max}^{\frac{1}{\rho}})\right)^\rho\quad \textrm{and} \quad \sigma_T=0,$$
where $\rho$, $\sigma_{\max}$ and $\sigma_{\min}$ are pre-defined constant numbers.  As for the negative samples, they are generated by the denoiser at time $t$. Then we start the denoising process defined in Eq.~(\ref{eq_sde}).
% Building upon the energy function, and in accordance with the stochastic sampling algorithm proposed in \cite{elucidate}, the reverse-time \emph{Stochastic Differential Equation (SDE)} is given by:
% \begin{align}
%     \mathrm{d}(s,a,s') = -\dot \sigma(t)\sigma(t)\nabla_{(s,a,s')}&\log q_t(s,a,s')\mathrm{d}t \notag \\
%     +\sqrt{2\beta(t)}&\sigma(t)\mathrm{d}\omega_t,
% \end{align}
% where $\beta(t)$ is the non-zero proportionality factor, with noise level $\sigma(0) = \sigma_{\max} > \sigma(1) >\cdots >\sigma(N) = 0$. 
The psuedo-code of EDIS is demonstrated in Appx.~\ref{appx_code}.

\section{Theoretical Analysis}\label{sec_theory}
In this section, we provide theoretical analysis on the suboptimality bound of our method and previous methods. We follow the assumptions of Fitted-Q-iteration~\cite{fqi}: Let $\gF$ be the Q function class that satisfies realizability: $Q^*\in \gF$. $\gF$ is closed under Bellman update: $\forall f\in \gF, \gT^\pi f\in \gF$, where $\gT^\pi f(s,a):=r(s,a)+\gamma \E_{s'\sim P(s,a)}[V^\pi_f(s')]$, $V^\pi_f(s'):=\E_{a'\sim \pi}f(s',a')$ and $\pi$ can be any policy. Additionally, we assume $\gF$ is $L$-Lipschitz.

In the offline phase, we initialize a value function $Q_0$. Based on $Q_0$, we perform online learning by collecting $n$ samples each iteration. The naive method, which relies solely on the information from $n$ collected samples to update the policy, is characterized by the following suboptimality bound.

\begin{theorem}\label{thm_no_model}
    Suppose the learned Q function after the offline phase satisfies $\left(\left\|Q^*-Q_0\right\|_\infty+\left\|Q^\pi-Q_0\right\|_\infty\right)\leq \epsilon_Q$, where $Q^*$ is the Q function of the optimal policy $\pi^*$. With probability at leaset $1-\delta$, the output policy $\pi$ after one iteration of online phase satisfies
    \begin{align*}
    &J(\pi^*)-J(\pi)\leq \frac{2\gamma\epsilon_Q}{1-\gamma}+\\
    &\quad \quad \frac{\sqrt{C_{\pi,\pi^*}}+1}{1+\gamma}\frac{1}{(1-\gamma)^{2.5}}\sqrt{\frac{56R^2_{\max}\log\frac{|\gF|^2}{\delta}}{3n}},
    \end{align*}
    where $C_{\pi,\pi^*}=\max_{s,a}\frac{\pi_1(a|s)}{\pi_2(a|s)}$, $n$ is the number of collected samples in one iteration.
\end{theorem}
Considering the performance bound after a single iteration is reasonable, as subsequent iterations can be analyzed by treating the output value function as the new initialization. 
% This is because the result for the next iteration can be derived by considering the output value function as the initialization and the final suboptimality bound will be favorable if the bound for every iteration is tight. 
The suboptimality after one iteration is characterized by the initialization error, augmented by a term of $O\left(\sqrt{\frac{1}{n}}\right)$. Notably, this error is nonnegligible because of the relatively small size of the online interaction samples.

A natural approach for leveraging offline data is augmenting the online dataset with the available offline data. To establish the suboptimality bound for such a method, we need to introduce the concept of concentratability coefficient.
\begin{definition}[Concentratability Coefficient]
    The concentratbility coefficient $C$ between state distribution $\nu$ and $\mu$ is defined as
    $\forall s\in \gS, \frac{\nu(s)}{\mu(s)}\leq C.$
\end{definition}
\vspace{-1mm}
Assuming the concentratability coefficient between the state distribution of current policy $d^\pi$ and offline dataset $D$ is $C_d$, we have the following theorem:

\begin{theorem}\label{thm_direct}
    Under the conditions of Thm.~\ref{thm_no_model}, by replaying offline data in the online iteration, the output policy $\pi$ after one iteration satisfies 
    $$
    \begin{aligned}
    &J(\pi^*)-J(\pi)\leq\frac{2\gamma\epsilon_Q}{1-\gamma}+\\
    &\quad \quad \frac{\sqrt{\widetilde{C_d}C_{\pi,\pi^*}}+1}{1+\gamma}\frac{1}{(1-\gamma)^{2.5}}\sqrt{\frac{56R^2_{\max}\log\frac{|\gF|^2}{\delta}}{3(n+N)}},
    \end{aligned}
    $$
    where $\widetilde{C_d}=\frac{(n+N)C_d}{N+nC_d}$, $N$ is the number of samples in offline dataset.
\end{theorem}
\vspace{-1mm}
The error bound is characterized by an order of $O\left(\sqrt{\frac{1}{n+N}}\right)$. Given the substantial difference in size, where $N$ is notably larger than $n$, replaying offline data shows advantages to the previous method. This result explains the improved performance resulting from online dataset augmentation. However, the introduction of the concentration coefficient $C_d$ in the new method reflects the consequences of the distribution shift issue. This coefficient can assume a significantly large value, particularly if the offline data fails to adequately cover a substantial portion of the state space. It detriments the potential performance improvement attainable through enhanced utilization of prior knowledge.
% The error bound has an order of $O\left(\sqrt{\frac{1}{n+N}}\right)$. Due to the considerable difference in size, with $N$ significantly larger than $n$, replaying offline data can be beneficial to the learning process, which explains why online dataset augmentation yields better performance. However, the new method introduces concentratability coefficient $C_d$, which is the consequence of distribution shift issue. This term can become significantly large if the offline data fails to cover enough portion of the state space, which limits the performance improvement from the better utilization of prior knowledge. 

Based on the diffusion model generator, our new method has access to synthetic samples $(\tilde s, \tilde a, \tilde r, \tilde{s}')$ during the online interaction phase. Let the distribution of $\tilde s$ be denoted as $\tilde{d}^\pi$, and the conditional distribution of $\tilde{s}'$ as $\widetilde{T}(\cdot|\tilde s, \tilde a)$. The state distribution error and model error of the diffusion model are defined as:$\epsilon_d=\left\|d^\pi-\tilde{d}^\pi\right\|$
    , and $\epsilon^d_m=\left\|T(\cdot|s,a)-\widetilde{T}(\cdot|s,a)\right\|_{\tilde{d}^\pi}$.

\begin{theorem}\label{thm_diffusion_error}
    Under the conditions of Thm.~\ref{thm_no_model}, we assume the data generator generates data  with state distribution error $\epsilon_d$ and model error $\epsilon_m$, then the output policy $\pi$ after one iteration of online phase satisfies
    \begin{align*}
    &J(\pi^*)-J(\pi)\leq \frac{2\gamma}{1-\gamma}\epsilon_Q + \\
    &\quad \ \left(\sqrt{C_{\pi,\pi^*}}+1\right)\left(\frac{1}{(1-\gamma)^2}R_{\max}\epsilon_d+\frac{\gamma L}{(1-\gamma)^2}\epsilon^d_m\right),
    \end{align*}
    where $L$ is Lipschitz constant of $\gF$.
\end{theorem}
\vspace{-1mm}
Here, we simplify the analysis by disregarding the generation errors associated with $\tilde a$ and $\tilde r$. This simplification is justified as modeling these errors is considerably more straightforward compared to modeling the error associated with $s$ and $\tilde s$. Additionally, considering these two errors leads to a similar analysis and does not alter the order of the error term.
According to the theorem, the error bound is not directly related to the number of samples, but it is highly dependent on the error of the diffusion model.
Based on the theory of supervised learning~\cite{vc}, the model error $\epsilon_m$ is $O\left(\sqrt{\frac{1}{n+N}}\right)$. The state distribution error $\epsilon_d$ is the distribution matching error, which is also $O\left(\sqrt{\frac{1}{n+N}}\right)$~\cite{nearly}. Therefore, our method exhibits the same convergence rate as the previous method without introducing the distribution shift issue identified by the concentration coefficient $C_d$.

Model-based method can also generate synthetic online samples. Assume the model error of the model is $\epsilon^t_m=\left\|T(\cdot|s,a)-\widetilde{T}(\cdot|s,a)\right\|_{d^\pi}$.

\begin{theorem}\label{thm_normal_error}
    Under the conditions of Thm.~\ref{thm_no_model}, we assume the model has an error $\epsilon^t_m$, then the output policy $\pi$ after one iteration of online phase satisfies
    $$
    \begin{aligned}
    &J(\pi^*)-J(\pi)\leq \frac{2\gamma}{1-\gamma}\epsilon_Q +\\
    &\quad \ \left(\sqrt{C_{\pi,\pi^*}}+1\right)\left(\frac{R_{\max}}{(1-\gamma)^3}\epsilon_m^t+\frac{\gamma L}{(1-\gamma)^2}\epsilon^t_m\right).
    \end{aligned}
    $$
\end{theorem}
\vspace{-1mm}
The error of model-based approach is irrelevant to the $C_d$ term, illustrating that incorporating prior knowledge about transitions from the offline dataset can address the issue of distribution shift. Nevertheless, this method is subject to a compounding error, with an error bound expressed as $O\left(\frac{1}{(1-\gamma)^3}\right)$. This limitation implies that model-based methods may not surpass the performance of model-free methods~\cite{moto}. In contrast, the error bound for EDIS is lower, at $O\left(\frac{1}{(1-\gamma)^2}\right)$, which effectively reduces the impact of compounding error.
In summary, EDIS overcomes all the issues previous methods have in theory. 

\begin{table*}[t]
\caption{Enhanced performance achieved by EDIS after 0.2M online fine-tuning on base algorithms Cal-QL and IQL. Each result is the average score over five random seeds $\pm$ standard deviation.}
\label{sample-table}
\vskip 0.15in
\begin{center}
% \footnotesize
\resizebox{.9\textwidth}{!}{
\begin{tabular}{@{}lcc|cc|cccc@{}}
\toprule
\multirow{2}{*}[-0.5ex]{\textbf{Dataset}} & \multicolumn{2}{c}{Cal-QL} & \multicolumn{2}{c}{IQL} & \multicolumn{2}{c}{Avg.} \\
\cmidrule(r){2-3} \cmidrule(l){4-5} \cmidrule(l){6-7}
& Base & Ours & Base & Ours & Base & Ours \\
\midrule
hopper-random-v2 & 17.6$\pm$3.1 & \textbf{98.1$\pm$12.3} & 10.0$\pm$1.7 & \textbf{12.1$\pm$4.0} & 13.8$\pm$2.4 & \textbf{55.1$\pm$8.2} \\
hopper-medium-replay-v2 & 102.2$\pm$4.6 & \textbf{109.9$\pm$0.8} & 99.0$\pm$4.7 & \textbf{101.1$\pm$1.6} & 100.6$\pm$4.7 & \textbf{105.5$\pm$1.2} \\
hopper-medium-v2 & 97.6$\pm$1.4 & \textbf{105.0$\pm$4.1} & 59.2$\pm$7.9 & \textbf{73.1$\pm$4.4} & 78.4$\pm$4.7 & \textbf{89.1$\pm$4.3} \\
hopper-medium-expert-v2 & 107.9$\pm$9.6 & \textbf{109.7$\pm$1.4} & 90.0$\pm$22.1 & \textbf{105.7$\pm$6.7} & 99.0$\pm$15.9 & \textbf{107.7$\pm$4.1} \\
halfcheetah-random-v2 & 74.8$\pm$3.2 & \textbf{86.3$\pm$1.8} & 36.7$\pm$3.0 & \textbf{38.9$\pm$1.9 }& 55.8$\pm$3.1 & \textbf{62.6$\pm$1.9} \\
halfcheetah-medium-replay-v2 & 76.6$\pm$1.2 & \textbf{86.7$\pm$1.4} & 45.6$\pm$0.4 & \textbf{47.1$\pm$0.3} & 61.1$\pm$0.8 & \textbf{66.9$\pm$0.9} \\
halfcheetah-medium-v2 & 72.3$\pm$2.1 & \textbf{83.9$\pm$1.0} & 48.6$\pm$0.2 & \textbf{49.8$\pm$0.2} & 60.5$\pm$1.2 & \textbf{66.9$\pm$0.6} \\
halfcheetah-medium-expert-v2 & 91.0$\pm$0.6 & \textbf{98.6$\pm$0.5} & \textbf{87.9$\pm$4.3} & 85.4$\pm$2.7 & 89.5$\pm$2.5 & \textbf{92.0$\pm$1.6} \\
walker2d-random-v2 & 15.1$\pm$3.5 & \textbf{61.6$\pm$12.6} & 6.5$\pm$0.7 & \textbf{16.2$\pm$2.9} & 10.8$\pm$2.1 & \textbf{38.9$\pm$7.8} \\
walker2d-medium-replay-v2 & 87.3$\pm$8.5 & \textbf{112.9$\pm$6.4} & 83.6$\pm$2.0 & \textbf{95.3$\pm$1.4} & 85.5$\pm$5.3 & \textbf{104.1$\pm$3.9} \\
walker2d-medium-v2 & 84.2$\pm$0.3 & \textbf{103.5$\pm$1.8} & 83.6$\pm$2.1 & \textbf{85.2$\pm$1.3} & 83.9$\pm$1.2 & \textbf{94.4$\pm$1.6} \\
walker2d-medium-expert-v2 & 111.1$\pm$0.6 & \textbf{118.5$\pm$4.0} & \textbf{108.9$\pm$2.9} & 107.5$\pm$4.5 & 110.0$\pm$1.8 & \textbf{113.0$\pm$4.3} \\
\addlinespace
\midrule
locomotion total & 937.7 & \textbf{1174.7} & 759.6 & \textbf{856.3} & 848.7 & \textbf{1015.5} \\
\midrule
\addlinespace
antmaze-umaze-v2 & 96.3$\pm$1.4 & \textbf{98.9$\pm$1.3} & 79.2$\pm$4.1 & \textbf{81.1$\pm$3.4} & 87.8$\pm$2.8 & \textbf{90.0$\pm$2.4} \\
antmaze-umaze-diverse-v2 & 93.4$\pm$4.6 & \textbf{95.9$\pm$2.8} & 51.3$\pm$4.5 & \textbf{66.7$\pm$5.0} & 72.4$\pm$4.6 & \textbf{81.3$\pm$3.9} \\
antmaze-medium-diverse-v2 & 81.4$\pm$3.9 & \textbf{89.3$\pm$4.8} & 75.6$\pm$1.9 & \textbf{81.8$\pm$4.8} & 78.5$\pm$2.9 & \textbf{85.6$\pm$4.8} \\
antmaze-medium-play-v2 & 86.8$\pm$1.6 & \textbf{93.9$\pm$2.7} & 81.0$\pm$2.2 & \textbf{86.2$\pm$1.3} & 83.9$\pm$1.9 & \textbf{90.1$\pm$2.0} \\
antmaze-large-play-v2 & 42.5$\pm$5.2 & \textbf{66.1$\pm$8.2} & 39.2$\pm$7.2 & \textbf{40.0$\pm$5.3} & 40.9$\pm$6.2 & \textbf{53.1$\pm$6.8} \\
antmaze-large-diverse-v2 & 42.3$\pm$2.2 & \textbf{57.1$\pm$2.8} & 45.0$\pm$8.7 & \textbf{52.1$\pm$2.6} & 43.7$\pm$5.5 & \textbf{54.6$\pm$2.7} \\
\addlinespace
\midrule
antmaze total & 442.7 & \textbf{501.2} & 371.3 & \textbf{407.9} & 407.0 & \textbf{454.6} \\
\midrule
relocate-human-v1 & -0.4$\pm$0.2 & \textbf{0.2$\pm$0.2} & 1.4$\pm$0.3 & \textbf{1.7$\pm$0.4} & 0.5$\pm$0.3 & \textbf{1.0$\pm$0.3} \\
pen-human-v1 & 68.4$\pm$8.7 & \textbf{95.6$\pm$6.2} & 91.2$\pm$5.2 & \textbf{103.6$\pm$4.2} & 79.8$\pm$7.0 & \textbf{99.6$\pm$5.2} \\
door-human-v1 & 0.1$\pm$0.2 & \textbf{58.4$\pm$17.6} & 20.7$\pm$4.6 & \textbf{25.6$\pm$3.3} & 10.4$\pm$2.4 & \textbf{42.0$\pm$10.5} \\
\midrule
adroit total & 68.1 & \textbf{154.2} & 113.3 & \textbf{130.9} & 90.7  & \textbf{142.3} \\
\midrule
\midrule
total & 1448.5 & \textbf{1830.1} & 1244.2 & \textbf{1395.1} & 1346.4 & \textbf{1612.4} \\
\bottomrule
\end{tabular}}
\label{tab:performancemain}
\end{center}
\vskip -0.1in
\end{table*}

\section{Experiments}

In this section, we empirically validate the effectiveness of EDIS. Sec.~\ref{performance} showcases the considerable performance enhancement achieved by EDIS when integrated with off-the-shelf offline-to-online algorithms. Sec.~\ref{sec:reason} investigates the reasons for this improvement, highlighting two key factors: i) the robust distribution modeling capability of diffusion models, and ii) the focus on modeling data distribution instead of the transition function, which effectively mitigates compounding errors. Sec.~\ref{sec:ablation} includes an ablation study on three energy functions to underscore their critical role in our algorithm. In this section, every experiment result is averaged over five random seeds.
% Our experiments are designed to answer the following questions:
% \begin{itemize}
%     \item Does integrating EDIS with off-the-shelf offline-to-online algorithms lead to enhanced sample efficiency in the fine-tuning phase?
%     \item Is the observed improvement attainable using a basic model-based method?
%     \item What role does each component of the energy function play in training the generator within EDIS?
%     \item How effectively does EDIS generate a distribution that is more conducive to online training with offline data compared to existing methods?
% \end{itemize}

\begin{figure*}
    % \vspace{-5em}
	\centering
	\makebox[\textwidth][c]{
	\includegraphics[width=0.85\paperwidth]{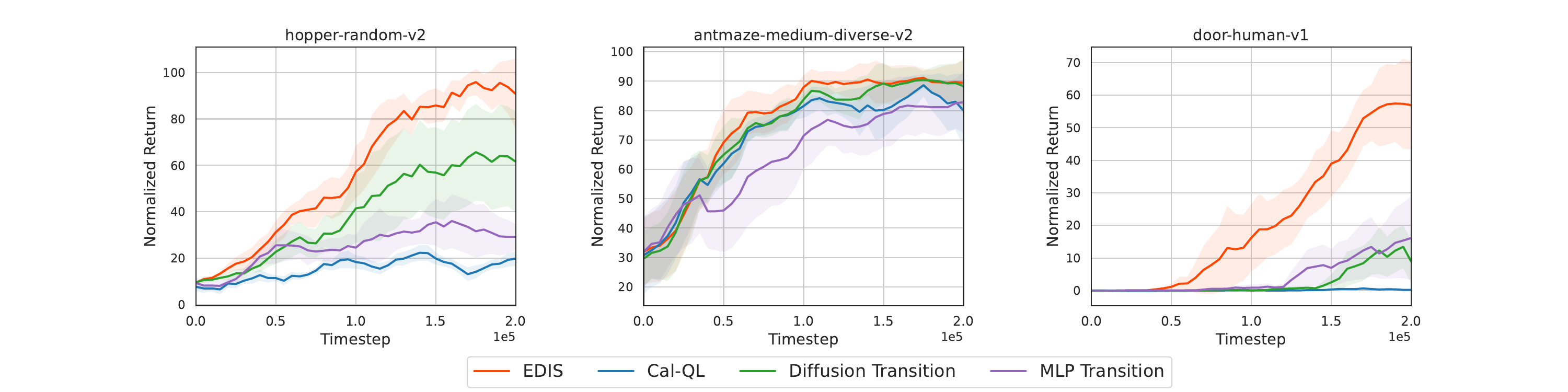}
	}
    \vspace{-1.5em}
 \caption{Comparison of EDIS and traditional model-based method. Diffusion transition and MLP transition mean the transition function is modeled by diffusion model and MLP respectively.}
	\label{fig:model}
    \vspace{-0.5em}
\end{figure*}
\subsection{Enhanced Performance Achieved by EDIS}
\label{performance}
We evaluate the performance of EDIS on three benchmark tasks from D4RL~\cite{d4rl}: MuJoCo Locomotion, AntMaze Navigation, and Adroit Manipulation. 
We implement EDIS on top of base algorithms Cal-QL~\cite{calql}, a state-of-the-art offline-to-online method that effectively calibrates over-conservatism of CQL~\cite{CQL}, and IQL~\cite{iql}, which adopts AWR-style policy learning in both phases. The 0.2M fine-tuning results after 1M pre-training are presented in Tab.~\ref{sample-table}. 

Notably, the integration of EDIS with Cal-QL yields a substantial 26.3\% improvement in overall performance, demonstrating significant enhancements in the Adroit domain and MuJoCo random datasets. The latter are characterized by a scarcity of successful demonstrations, creating a substantial gap between offline and online distributions that poses a challenge to previous methods. EDIS addresses this issue by directly generating online samples, overcoming the distribution shift issue. Even combined with IQL, where the conservatism cannot be entirely dropped during online fine-tuning due to its reliance on in-sample data, EDIS still achieves an average performance improvement of 12.1\%.  For more information on implementation details, please refer to Appx.~\ref{appendix:experiment}.

\vspace{-0.5em}

\subsection{Comparisons between EDIS and Basic Model-based Methods}
\label{sec:reason}
While EDIS demonstrates superior performance on the D4RL benchmarks, a natural question arises: what contributes to this improvement? We hypothesize two main contributors to EDIS's success: firstly, the superior representational capabilities of diffusion models, and secondly, the innovative modeling approach that focuses on the distribution of $(s,a,s')$ tuples instead of modeling the transition function as is common in traditional model-based methods.  To substantiate our hypothesis, we have undertaken further experiments involving two models tasked with representing the transition function of the environments. The models have equivalent capacity as EDIS (i.e., having the same number of parameters), allowing for a direct comparison. These models are then used to rollout current policies within the learned model, thereby augmenting online data. One model employs a multi-layer perceptron (MLP) structure, while the other utilizes a diffusion model. The base algorithm is Cal-QL and the results are shown in Fig.~\ref{fig:model}. It can be seen that the MLP model does not exhibit enhanced performance compared to the original algorithm. According to the theory outlined in Sec.~\ref{sec_theory}, this suggests that the performance drop induced by distribution shift and MLP model error is similar.

In contrast, the diffusion transition model shows some improvement, indicating that diffusion model suffers smaller model error than MLP model because of its superior distribution modeling ability. Despite this, simply employing a powerful diffusion model in conjunction with traditional data augmentation methods could not match the performance achieved by EDIS. This is attributed to the compounding error introduced by rolling out policies in a misspecified transition model. The unique data generation approach of EDIS, which directly generates online samples, proves crucial in achieving its remarkable performance. To visualize the effect of compounding error reduction, we compare the state distribution generated by these models. The result is defered to Appx.~\ref{appx_visualize} because of space limitation.
In summary, the performance enhancement observed with EDIS stems not only from the capability of the diffusion model but also from its distinctive data generation methodology.

\vspace{-0.5em}

\subsection{Ablation for Energy Functions in EDIS}

\label{sec:ablation}

\begin{table*}[t]
\caption{Divergence comparisons for energy function ablation study (lower is better). Each result is the average score over five random seeds $\pm$ standard deviation.}
\label{tab:ablationenergy}
\vskip 0.15in
\begin{center}
\footnotesize
\begin{tabular}{@{}lcc|cc|cccc@{}}
\toprule
\multirow{2}{*}[-0.5ex]{\textbf{Dataset}} & \multicolumn{2}{c}{State Divergence} & \multicolumn{2}{c}{Action Divergence} & \multicolumn{2}{c}{Transition Divergence} \\
\cmidrule(r){2-3} \cmidrule(l){4-5} \cmidrule(l){6-7}
& w/o energy & w/ energy & w/o energy & w/ energy & w/o energy & w/ energy \\
\midrule
hopper-radnom-v2 & 0.85$\pm$0.02 & \textbf{0.73$\pm$0.03} & 0.51$\pm$0.04 & \textbf{0.39$\pm$0.02} & 0.69$\pm$0.03 & \textbf{0.66$\pm$0.04} \\

% halfcheetah-medium-replay-v2 & 0.62$\pm$0.04  & \textbf{0.45$\pm$0.04} & 0.43$\pm$0.15 & \textbf{0.31$\pm$0.03} & 1.01$\pm$0.08 & \textbf{0.88$\pm$0.07} \\
antmaze-medium-diverse-v2 & 0.98$\pm$0.00 & \textbf{0.91$\pm$0.02} & 0.38$\pm$0.08 & \textbf{0.27$\pm$0.08} & 0.75$\pm$0.14 & \textbf{0.64$\pm$0.14} \\
door-human-v1 & 0.40$\pm$0.04 & \textbf{0.24$\pm$0.02} & 0.54$\pm$0.07 & \textbf{0.41$\pm$0.08} & 2.57$\pm$0.06 & \textbf{2.52$\pm$0.05} \\ 
\bottomrule
\end{tabular}
\label{tab:performance}
\end{center}
\vskip -0.1in
\end{table*}

\begin{figure*}
	\centering
	\makebox[\textwidth][c]{
	\includegraphics[width=0.85\paperwidth]{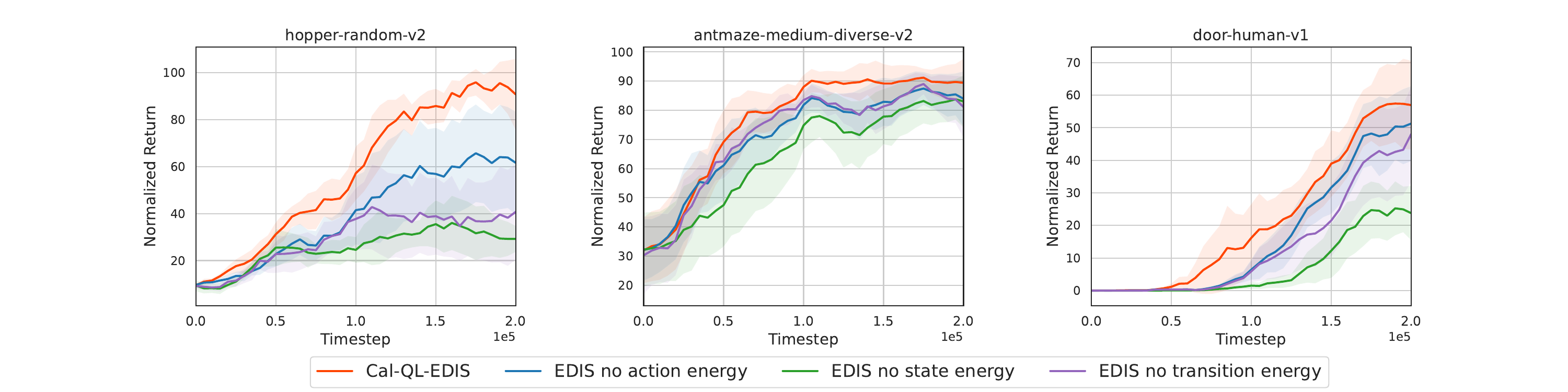}
	}
     \vspace{-2em}
	\caption{Energy Module Ablation Study of EDIS}
	\label{fig:ablation}
\end{figure*}

The energy functions, as described in Sec.~\ref{thm_energy}, align each component in the generated distribution with their online counterparts, including the state distribution, the action distribution under the current policy, and the transition fidelity.
We conduct ablation studies on each energy function by omitting its guidance during the reverse-time \emph{SDE} in hopper-random-v2, antmaze-medium-diverse-v2, and door-human-v1. We compare the divergence between the distribution generated by diffusion models and the real environments. For the state divergence, we calculate the JS divergence between the two distributions. The details of JS divergence and its calculation state distribution are shown in Appx.~\ref{appx_ablation}. For action divergence and state divergence, we calculate the mean square loss of the action or next state generated by the diffusion model and the current policy or the real dynamics.
As shown in Tab.~\ref{tab:ablationenergy}, the guidance of the energy function effectively decreases the divergence from the online sample distribution, real-world transitions, and actions induced by the policy interacting with the environment.

Furthermore, the results in Fig.~\ref{fig:ablation} demonstrate that a larger divergence of the three parts detriments the final performance.
% , either marginally or significantly. 
Particularly, in environments like hopper-random-v2, EDIS struggles to excel without the state energy function, demonstrating the ineffectiveness of merging unrelated offline data directly into the online fine-tuning phase.

\section{Related Work}
In this section, we provide a brief summary of related work, the complete version is in Appx.~\ref{appx_related_work}.

\textbf{Offline-to-online Reinforcement Learning.}
Offline-to-online reinforcement learning methods are developed to mitigate the dichotomy between the costly exploration in online RL and the typically suboptimal performance of offline RL.
The learning process is typically divided into two phases: Warming-up the policy and value functions in the offline phase and using them as initialization in the online phase~\cite{ awac, oncefamily, calql, balanced}. These approaches often employ offline RL methods based on policy constraints or pessimism in the offline phase~\cite{bcq, td3bc, CQL}. 
% 'warm up' the policy and value functions to prime them for the online fine-tuning phase \cite{awac, adaptivebc, oncefamily}. 
However, the conservatism conflicts with the online phase and may induce performance degradation.
% A prevalent issue these methods tackle is the over-conservatism derived from the pre-training phase, which may not adequately consider the distributional shifts that occur in online environments. 
Various strategies have been implemented to tackle this issue,
% circumvent the deterioration of the initial pre-trained policy, 
such as policy expansion \cite{pex}, value function calibration \cite{calql}, Q-ensemble techniques \cite{balanced}, and constraint methods \cite{awac, iql, proto}. 
% Collectively, these approaches seek a compromise between reaching high asymptotic performance and reducing the initial performance degradation.
Despite the advancements, there has been less focus on the crucial aspect of integrating useful data during the fine-tuning phase to enhance training efficiency. Standard practices include enriching the replay buffer with offline data \cite{calql, pex}, adopting balanced sampling methods for managing both online and offline data sources \cite{balanced, efficientonline}, or building models to performance branch rollout~\cite{moto}. However, directly replaying the offline data causes a distribution shift, and adopting balanced sampling methods introduces large variance~\cite{dualdice} while rolling in the built model suffers from compounding error~\cite{MBPO}. 
% fall short of a comprehensive solution for effectively capitalizing on the wealth of knowledge in offline datasets.
In contrast, our work breaks new ground by proposing a diffusion-based generator specifically designed to generate useful samples by thoroughly leveraging prior knowledge in offline data.

\textbf{Diffusion Model in Reinforcement Learning.}
Diffusion models have demonstrated exceptional capabilities in modeling distribution~\citep{SahariaCSLWDGLA22, NicholDRSMMSC22, NicholD21}. Within the reinforcement learning research, Diffuser~\citep{diffuser} uses a diffusion model as a trajectory generator and learns a separate return model to guide the reverse diffusion process toward samples of high-return trajectories. Decision Diffuser~\citep{DecisionDiffuser} introduces conditional diffusion with reward or constraint guidance for decision-making tasks, further boosting Diffuser’s performance.
% Diffusion-QL~\cite{DiffusionQL} tracks the gradients of the actions sampled from the behavior diffusion policy to guide generated actions to high Q-value area. SfBC~\cite{sfbc} and Diffusion-QL both employ the technique of resampling actions from multiple behavior action candidates using predicted Q-values as sampling weights.
Expanding the application of diffusion models, SYNTHER~\cite{synther} focuses on leveraging the diffusion model for upsampling data in both online and offline reinforcement learning scenarios. Recently, several concurrent work also investigates generating samples with certain data distribution by diffusion models.
PolyGRAD~\cite{policyguidedtrajdiff} and PGD~\cite{policyguideddiffusion} embed the policy for classifier-guided trajectory generation, aiming at on-policy world modeling. However, they still use the diffusion model to model the transition function rather than the data distribution directly, which does not eliminate the issue of compounding error. Finally, in the context of offline-to-online reinforcement learning, our research pioneers the utilization of diffusion-based models to actively generate valuable samples. This distinction from the passive reuse of offline data marks EDIS as a novel approach, emphasizing the active role of diffusion models in sample generation for offline-to-online RL.

\section{Conclusion}
This study underscores the crucial role of harnessing offline data effectively to augment online fine-tuning. We introduce Energy-guided Diffusion Sampling (EDIS), a powerful approach leveraging a state-of-the-art generative model with energy guidance. In contrast to existing strategies that primarily rely on initializing with offline-derived policies and Q functions or replaying offline data, EDIS takes a proactive stance by generating synthetic data within the online data distribution, drawing information from all available data sources.
Theoretical analysis verifies that EDIS outperforms prior methods in addressing distribution shift challenges and effectively mitigates compounding errors. As a versatile solution, EDIS seamlessly integrates with prevalent offline-to-online frameworks. When combined with methods like Cal-QL and IQL, EDIS showcases substantial performance enhancements.
This superior generative modeling capabilities of EDIS mark it as a scalable approach, providing promising potential for designing data-efficient learning strategies in more complex environments with high-dimensional state or action spaces.

\section*{Acknowledgement}
This work is supported by the National Science Foundation of China (61921006). The authors thank anonymous reviewers for their helpful discussions and suggestions for
improving the article.

\section*{Impact Statement}
This paper presents work whose goal is to advance the field of Machine Learning. There are many potential societal consequences of our work, none of which we feel must be specifically highlighted here.

\bibliography{example_paper}
\bibliographystyle{icml2024}

%%%%%%%%%%%%%%%%%%%%%%%%%%%%%%%%%%%%%%%%%%%%%%%%%%%%%%%%%%%%%%%%%%%%%%%%%%%%%%%
%%%%%%%%%%%%%%%%%%%%%%%%%%%%%%%%%%%%%%%%%%%%%%%%%%%%%%%%%%%%%%%%%%%%%%%%%%%%%%%
% APPENDIX
%%%%%%%%%%%%%%%%%%%%%%%%%%%%%%%%%%%%%%%%%%%%%%%%%%%%%%%%%%%%%%%%%%%%%%%%%%%%%%%
%%%%%%%%%%%%%%%%%%%%%%%%%%%%%%%%%%%%%%%%%%%%%%%%%%%%%%%%%%%%%%%%%%%%%%%%%%%%%%%
\newpage
\appendix
\onecolumn
\section{Proofs and additional Theory}

\subsection{Proof of Thm.~\ref{thm_energy}}

Because $e^{\gE_1(s)}\propto \frac{p_\theta(s)}{d^\pi(s)}$, $e^{\gE_2(a|s)}\propto \frac{p_\theta(a|s)}{\pi(a|s))}$, $e^{\gE_3(s'|s,a)}\propto \frac{p_\theta(s'|s,a)}{T^\pi(s'|s,a))}$, we have $e^{\gE_1(s)}= k_1\frac{p_\theta(s)}{d^\pi(s)}$, $e^{\gE_2(a|s)}=k_2\frac{p_\theta(a|s)}{\pi(a|s)}$, $e^{\gE_3(s'|s,a)}=k_3\frac{p_\theta(s'|s,a)}{T(s'|s,a)}$, where $k_1$, $k_2$ and $k_3$ are arbitrary positive constants.
\begin{align*}
    &\quad \ e^{-(\gE_{1}(s)+\gE_{2}(a|s)+\gE_{3}(s'|s,a))}\\
    &=k_1k_2k_3\frac{d^\pi(s)\pi(a|s)T(s'|s,a)}{p_\theta(s)p_\theta(a|s)p_\theta(s'|s,a)}\\
    &=k_1k_2k_3\frac{q_\pi(s,a,s')}{p_\theta(s,a,s')},
\end{align*}
where the last equation comes from $q_\pi(s,a,s')=d^\pi(s)\pi(a|s)T(s'|s,a)$ and $p_\theta(s,a,s')=p_\theta(s)p_\theta(a|s)p_\theta(s'|s,a)$. Because $k_1k_2k_3>0$, we conclude $e^{-(\gE_{1}(s)+\gE_{2}(a|s)+\gE_{3}(s'|s,a))}\propto \frac{q_\pi(s,a,s')}{p_\theta(s,a,s')}$, then $p_\theta(s,a,s')e^{-\gE(s,a,s')}\propto q_\pi(s,a,s')$, and thus Eq.~(\ref{eq_modified_dist}) holds.

\subsection{Proof of Thm.~\ref{thm_no_model}}\label{appx_51}
Before proof this theorem, we introduce some new notations. 
$$\gT^\pi f(s,a):=r(s,a)+\gamma \E_{s'\sim P(s,a)}[V^\pi_f(s')],$$ where $V^\pi_f(s'):=\E_{a'\sim \pi}f(s',a')$. Recursively define $(\gT^\pi)^j=(\gT^\pi)(\gT^\pi)^{j-1}$.

Similarly, we define $$\gT f(s,a):=r(s,a)+\gamma \E_{s'\sim P(s,a)}[V_f(s')],$$
where $V_f(s'):=\max_{a'}f(s',a')$.

Given a function $f:X\mapsto \sR$, define its norm $\|\cdot\|_\mu$ with respect to distribution $\mu$ as:
$$\|f\|_\mu=\sqrt{\sum_xf^2(x)\mu(x)}.$$

Suppose $\nu(s)$ is a state state distribution $\pi(a|s)$ is an action distribution based on given state $s$, then $\nu\times \pi(s,a)=\nu(s)\pi(a|s)$.

\begin{lemma}[Theorem 1 of~\cite{lfiw}]\label{lemma_contraction}
    The Bellman operator $\gT$ is a $\gamma$-contraction with respect to the $\|\cdot\|_{d^\pi}$ norm, i.e., 
    $$\left\|\gT^\pi f-\gT^\pi f'\right\|_{d^\pi}\leq \gamma \left\|f-f'\right\|_{d^\pi}, \quad \forall f, f'\in \gF.$$
\end{lemma}

\begin{lemma}\label{lemma_ratio}
    Let $\nu$ be any admissible state distribution, $\pi_1$ and $\pi_2$ be two policies, and $C_{\pi_1,\pi_2}=\max_{s,a}\frac{\pi_1(a|s)}{\pi_2(a|s)}$, $\left\|\cdot\right\|_{\nu \times\pi_1}\leq \sqrt{C_{\pi_1,\pi_2}}\left\|\cdot\right\|_{\nu\times\pi_2}$.
\end{lemma}

\begin{proof}
    For any function $f:\gS\times \gA\to \sR$, we have
    \begin{align*}
        \left\|f\right\|_{\nu\times\pi_1}&=\left(\sum_{s,a}|f(s,a)|^2\nu(s)\pi_1(a|s)\right)^{1/2}\\
        &\leq \left(\sum_{s,a}|f(s,a)|^2 C_{\pi_1,\pi_2}\nu(s)\pi_2(a|s)\right)^{1/2}\\
        &=\sqrt{C_{\pi_1,\pi_2}}\left(\sum_{s,a}|f(s,a)|^2\nu(s)\pi_2(a|s)\right)^{1/2}\\
        &=\sqrt{C_{\pi_1,\pi_2}}\left\|f\right\|_{\nu\times\pi_2}.
    \end{align*}
\end{proof}

To simplify the notation, we use $\widehat{\gT}^\pi$ to denote the Bellman operator under limited data, i.e.,
$$\widehat{\gT}^\pi f=\argmin_g\sum_{i=1}^N(g-r-\gamma V^\pi_f)^2.$$
\begin{lemma}\label{lemma_sample}
    If $f\in \gF$, with probability at least $1-\delta$, we have
    $$\left\|(\widehat{\gT}^\pi)^j f - (\gT^\pi)^j f\right\|_{d^\pi} \leq \frac{1-\gamma^j}{1-\gamma}\frac{56V^2_{\max}\log\frac{|\gF|^2}{\delta}}{3n}.$$
\end{lemma}

\begin{proof}
The proof is based on Lem.~16 of~\cite{fqi}.

First, we consider $\left\|\widehat{\gT}^\pi f - \gT^\pi f\right\|_{d^\pi}$. Define
\begin{align*}
&L_{d^\pi}(g, h)=\E_{s,a\sim d^{\pi}}(g-r-\gamma V^\pi_h)^2\\
&L_{D}(g, h)=\E_{s,a\sim D}(g-r-\gamma V_h^\pi )^2
\end{align*}
Then we have $\gT f=\argmin_g L_{d^\pi}(g,f)$ and $\widehat{\gT} f=\argmin_g L_{D}(g,f)$.

To simplify the notations, we define
$$X(g,f,g^*):=(g(s,a)-r-\gamma V^\pi_f(s'))^2-(g^*(s,a)-r -V_f^\pi(s'))^2$$

Next, we bound the variance of $X(g,f,g^*)$,
    \begin{align*}
        \sV\left[X(g,f,g^*)\right]&\leq \E\left[\left\|X^2(g,f,g^*)\right\|\right]\\
        &=\E\left[\left((g(s,a)-r-\gamma V^\pi_f(s'))^2-(g^*(s,a)-r -V_f^\pi(s'))^2\right)\right]\\
        &=\E\left[(g(s,a)-g^*(s,a))^2(g(s,a)+g^*(s,a)-2r-2\gamma V_f^\pi(s'))^2\right]\\
        &\leq 4V_{\max}^2\E\left[(g(s,a)-g^*(s,a))^2\right]\\
        &=4V_{\max}^2\left\|g-g^*\right\|^2_{d^\pi}
    \end{align*}

Note that
\begin{align*}
    \left\|g-g^*\right\|^2_{d^\pi}&\overset{(a)}{\leq} 2(\|g-r-\gamma V_f^\pi\|^2_{d^\pi}+\|r+\gamma V_f^\pi -g^*\|_{d^\pi}^2)\\
    &=2\E[X(g,f,g^*)],
\end{align*}

where (a) holds because $(a+b)^2\leq 2(a^2+b^2)$.

Finally, we apply Bernstein's inequality and union bound over all $f\in \gF$. With probability at least $1-\delta$, we have
    \begin{align*}
        &\quad \ \E[X(g,f,g^*)-\frac{1}{n}\sum_{i=1}^nX_i(g,f,g^*)\\
        &\leq \sqrt{\frac{2\sV\left[X(g,f,g^*)\right]\log \frac{|\gF|^2}{\delta}}{n}} + \frac{4V^2_{\max}\log \frac{|\gF|^2}{\delta}}{3n}\\
        &\leq \sqrt{\frac{ 16V_{\max}^2\E[X(g,f,g^*)\log \frac{|\gF|^2}{\delta}}{n}} + \frac{4V^2_{\max}\log \frac{|\gF|^2}{\delta}}{3n}
    \end{align*}
    Since $\widehat{\gT}^\pi f$ minimizes $L_D(\cdot,f)$, it also minimizes $\frac{1}{n}X_i(\cdot,f,g^*)$. This is because the two objectives only differ by a constant $\gL_D(\cdot, f)$. Therefore, we have  
    $$\frac{1}{n}\sum_{i=1}^nX_i(\widehat{\gT}^\pi f, f, \gT^\pi f)\leq \frac{1}{n}\sum_{i=1}^nX_i(\gT^\pi f, f, \gT^\pi f)=0.$$
    Then
    $$\E[X(\widehat{\gT}^\pi,f,\gT^\pi f)\leq \sqrt{\frac{ 16V_{\max}^2\E[X(g,f,g^*)\log \frac{|\gF|^2}{\delta}}{n}} + \frac{4V^2_{\max}\log \frac{|\gF|^2}{\delta}}{3n}.$$
    Solving for the quadratic formula, we have
    $$\E[X(\widehat{\gT}^\pi,f,\gT^\pi f)\leq \frac{56V^2_{\max}\log\frac{|\gF|^2}{\delta}}{3n}.$$

    Noticing that
    \begin{equation}\label{eq_1-re}
    \begin{aligned}
        \left\|\widehat{\gT}^\pi f-\gT^\pi f\right\|^2_{d^\pi}&=\gL_{d^\pi}(\widehat{\gT}^\pi f,f)+\gL_{d^\pi}(\gT^\pi f,f)\\
        &=\E[X(\widehat{\gT}^\pi f,f,\gT^\pi f)]\leq \epsilon,
    \end{aligned}
    \end{equation}
    where $\epsilon=\frac{56V^2_{\max}\log\frac{|\gF|^2}{\delta}}{3n}$.

    For $\left\|(\widehat{\gT}^\pi)^2 f-(\gT^\pi)^2 f\right\|^2_{d^\pi}$, note that
    \begin{align*}
        \left\|(\widehat{\gT}^\pi)^2 f-(\gT^\pi)^2 f\right\|^2_{d^\pi}&=\left\|\widehat{\gT}^\pi(\widehat{\gT}^\pi f)-\gT^\pi(\gT^\pi f)\right\|^2_{d^\pi}\\
        &\leq \left\|\widehat{\gT}^\pi(\widehat{\gT}^\pi f)-\gT^\pi(\widehat{\gT}^\pi f)\right\|^2_{d^\pi}+\left\|\gT^\pi(\widehat{\gT}^\pi f)-\gT^\pi(\gT^\pi f)\right\|^2_{d^\pi}\\
        &\overset{(a)}{\leq} \epsilon+\left\|\gT^\pi(\widehat{\gT}^\pi f)-\gT^\pi(\gT^\pi f)\right\|^2_{d^\pi}\\
        &\overset{(b)}{\leq} \epsilon+\gamma^2 \left\|\widehat{\gT}^\pi f-\gT^\pi f\right\|^2_{d^\pi}\\
        &\overset{(c)}{\leq} \epsilon+\gamma^2\epsilon,
    \end{align*}
where (a) and (c) use Eq.~(\ref{eq_1-re}), (b) use Lem.~\ref{lemma_contraction}.

Recursively, we have
$$\left\|(\widehat{\gT}^\pi)^2 f-(\gT^\pi)^2 f\right\|^2_{d^\pi}\leq \sum_{i=0}^j \gamma^{2i} \epsilon=\frac{1-\gamma^{2j}}{1-\gamma^2}\frac{56V^2_{\max}\log\frac{|\gF|^2}{\delta}}{3n}.$$
\end{proof}

\begin{lemma}[Lemma 14 in \cite{fqi}]\label{lemma_max}
Assume $f, f'\in \gF$ and define $\pi_{f,f'}:=\argmax_{a\in \gA}\max \{f(s,a),f'(s,a)\}$. Then we have $\forall \nu\in \Delta(\gS\times \gA)$, 
    $$\left\|\gT f- \gT f'\right\|_\nu\leq \gamma \left\|f-f'\right\|_{\nu\times \pi_{f,f'}}.$$
\end{lemma}

Suppose $\pi$ is induced by $Q_0$, 
\begin{equation}\label{eq_decompose1}
\begin{aligned}
    J(\pi^*)-J(\pi)&=\sum_{t=0}^\infty \E_{s\sim d^{\pi}}\gamma^t [V^*(s)-Q^*(s,\pi)]\\
    &\leq \sum_{t=0}^\infty \gamma^t \E_{s\sim d^{\pi}} [Q^*(s,\pi^*)-\widetilde Q(s,\pi^*)+\widetilde Q(s,\pi)-Q^*(s,\pi)]\\
    &\leq \sum_{t=0}^\infty \left(\sum_{s,a}d^{\pi}(s)\pi^*(a|s)\left|Q^*(s,a)-\widetilde Q(s,a)\right|+\sum_{s,a}d^{\pi}(s)\pi(a|s)\left|Q^*(s,a)-\widetilde Q(s,a)\right|\right)\\
    &\leq \sum_{t=0}^\infty \gamma^t \left(\left\|Q^*-\widetilde Q\right\|_{d^{\pi}\times \pi^*}+\left\|Q^*-\widetilde Q\right\|_{d^{\pi}}\right)
\end{aligned}
\end{equation}

\begin{equation}\label{eq_decompose2}
\begin{aligned}
    &\quad \ \left\|Q^*-\widetilde{Q}\right\|_\nu\\
    &=\left\|Q^*-\gT Q_0+\gT Q_0-Q^{\pi}+Q^{\pi}-\widetilde{Q}\right\|_\nu\\
    &\leq \left\|Q^*-\gT Q_0\right\|_\nu + \left\|\gT Q_0-Q^{\pi}\right\|_\nu+\left\|Q^{\pi}-\widetilde{Q}\right\|_\nu\\
    &=\left\|Q^*-\gT Q_0\right\|_\nu +\left\|Q^{\pi}-(\gT^{\pi})^\infty Q_0\right\|_\nu +\left\|(\gT^{\pi})^\infty Q_0-\widetilde{Q}\right\|_\nu+ \left\|\gT Q_0-Q^{\pi}\right\|_\nu\\
    &=\left\|\gT Q^*-\gT Q_0\right\|_\nu+\left\|Q^{\pi}-(\gT^{\pi})^\infty Q_0\right\|_\nu+\left\|(\gT^{\pi})^\infty Q_0-\widetilde{Q}\right\|_\nu+ \left\|\gT Q_0-Q^{\pi}\right\|_\nu\\
    &\overset{(a)}{\leq}\gamma \left\|Q^*-Q_0\right\|_{\nu\times \pi_{Q^*,Q_0}}+\left\|Q^{\pi}-(\gT^{\pi})^\infty Q_0\right\|_\nu+\left\|(\gT^{\pi})^\infty Q_0-\widetilde{Q}\right\|_\nu+ \left\|\gT Q_0-Q^{\pi}\right\|_\nu\\
    &= \gamma \left\|Q^*-Q_0\right\|_{\nu\times \pi_{Q^*,Q_0}}+\left\|Q^{\pi}-(\gT^{\pi})^\infty Q_0\right\|_\nu+\left\|(\gT^{\pi})^\infty Q_0-(\widehat{\gT}^\pi)^\infty Q_0\right\|_\nu+ \left\|\gT Q_0-Q^{\pi}\right\|_\nu,
    % +\left\|(\widehat{\gT}^{\pi})^\infty Q_0-(\widetilde{\gT}^\pi)^\infty Q_0\right\|_\nu
\end{aligned}
\end{equation}
where (a) uses Lem.~\ref{lemma_max}. The last equality comes from the fact that $\widetilde{Q}=(\widehat{\gT}^\pi)^\infty Q_0$.

When $\nu=d^{\pi}\times \pi^*$, we can use Lem.~\ref{lemma_ratio} to obtain
$$
\left\|\cdot\right\|_\nu\leq \sqrt{C_{\pi, \pi^*}}\left\|\cdot\right\|_{d^\pi}.
$$
When $\nu=d^\pi\times \pi_{Q^*,Q_0}$, note that $\pi_{Q^*,Q_0}$ is more similar to $\pi$ than $\pi^*$, so
$$
\left\|\cdot\right\|_{\nu\times \pi_{Q^*,Q_0}}\leq \sqrt{C_{\pi, \pi^*}}\left\|\cdot\right\|_{d^\pi}.
$$

Therefore, we only need to consider how to bound these terms when $\nu=d^{\pi}$.

According to Lem.~\ref{lemma_contraction}, 
$$\|\gT Q_0-Q^\pi\|_{d^\pi}\overset{(b)}{=}\|\gT Q_0-\gT^\pi Q^\pi\|_{d^\pi}=\|\gT^\pi Q_0-\gT^\pi Q^\pi\|_{d^\pi}\overset{(c)}{\leq} \gamma\|Q^\pi-Q_0\|_{d^{\pi}},$$

where (b) comes from $\pi(a|s)=\argmax Q_0(s,a)$, (c) uses Lem.~\ref{lemma_contraction}.

Let $j\to \infty$, we can apply Lem.~\ref{lemma_sample} to term $\left\|(\gT^{\pi})^\infty Q_0-(\widehat{\gT}^\pi)^\infty Q_0\right\|_{d^\pi}$, 

$$\left\|(\gT^{\pi_k})^\infty Q_0-(\widehat{\gT}^\pi)^\infty Q_0\right\|_{{d^{\pi}}}\leq \epsilon',$$
where $\epsilon'=\sqrt{\frac{1}{1-\gamma^2}\frac{56V^2_{\max}\log\frac{|\gF|}{\delta}}{3n}}$ and the inequality holds with probability at least $1-\delta$. 

For the term $\left\|Q^{\pi}-(\gT^{\pi})^\infty Q_0\right\|_{d^\pi}$, note that

\begin{align*}
    \left\|Q^{\pi}-(\gT^{\pi})^\infty Q_0\right\|_{d^\pi}&=\left\|\gT^\pi Q^{\pi}-\gT^\pi (\gT^{\pi})^\infty Q_0\right\|_{d^\pi}\\
    &\leq \gamma \left\|Q^{\pi}-(\gT^{\pi})^\infty Q_0\right\|_{d^\pi}
\end{align*}

The inequality uses Lem.~\ref{lemma_contraction}. Recursively, we get $\left\|Q^{\pi}-(\gT^{\pi})^\infty Q_0\right\|_{d^\pi}=0$.

To sum up, we have
\begin{align*}
    J(\pi^*)-J(\pi)&\leq \sum_{t=0}^\infty \gamma^t \left(\left\|Q^*-\widetilde Q\right\|_{d^{\pi}\times \pi^*}+\left\|Q^*-\widetilde Q\right\|_{d^{\pi}}\right)\\
    &\leq \frac{2\gamma}{1-\gamma}\left(\left\|Q^*-Q_0\right\|_{d^\pi\times \pi_{Q^*,Q_0}}+\left\|Q^\pi-Q_0\right\|_{d^\pi}\right)+\frac{\sqrt{C_{\pi,\pi^*}}+1}{1+\gamma}\frac{1}{(1-\gamma)^{2.5}}\sqrt{\frac{56R^2_{\max}\log\frac{|\gF|^2}{\delta}}{3n}}\\
    &\leq \frac{2\gamma}{1-\gamma}\left(\left\|Q^*-Q_0\right\|_\infty+\left\|Q^\pi-Q_0\right\|_\infty\right)+\frac{\sqrt{C_{\pi,\pi^*}}+1}{1+\gamma}\frac{1}{(1-\gamma)^{2.5}}\sqrt{\frac{56R^2_{\max}\log\frac{|\gF|^2}{\delta}}{3n}}.
\end{align*}
Then we conclude the proof.

\subsection{Proof of Thm.~\ref{thm_direct}}\label{appx_52}

Let the state distribution of the combined buffer as $d(s)$, the concentratability coefficient between $d^\pi$ and $d$ as $C_d$, i.e., $C_d=\max_{s}\frac{d^\pi(s)}{d(s)}$. We introduce the following lemmas. 

\begin{lemma}\label{lemma_combine}
    $D_1$, $D_2$ are two dataset containing $N$ samples and $n$ samples, respectively. If the distribution of $D_1$ is $\mu$ and $D_2$ is $\nu$, the concentratabilty coefficient between $\nu$ and $\mu$ is $C$, then the concentratability coefficient between $\nu$ and combined buffer $\rho$ is $\frac{(N+n)C}{N+nC}$.
\end{lemma}

\begin{proof}
    The distribution of combined buffer is $\rho=\frac{N\mu(s)+n\nu(s)}{N+n}$.

    \begin{align*}
        \max_{s}\frac{\nu(s)}{\rho(s)}&=\max_{s}\frac{(N+n)\nu(s)}{N\mu(s)+n\nu(s)}\\
        &=\max_{s}\frac{N+n}{N\frac{\mu(s)}{\nu(s)}+n}\\
        &=\frac{(N+n)C}{N+nC}.
    \end{align*}
\end{proof}

\begin{lemma}[Lemma 12 in \cite{fqi}]\label{lemma_con}
    Let $\mu(s)$ be any admissible distribution, $C$ be the concentratablility coefficient of $\mu(s)$ and $\nu(s)$,  then $\|\cdot\|_\nu\leq \sqrt{C}\|\cdot\|_\mu$.
\end{lemma}

Note that Eq.~(\ref{eq_decompose1}) and (\ref{eq_decompose2}) hold whatever the distribution of training data is. By replacing $\nu$ in Eq.~(\ref{eq_decompose2}) with $d\times \pi^*$ and $d\times \pi_{Q^*,Q_0}$, and apply Lem.~\ref{lemma_con}, Lem.~\ref{lemma_contraction} and Lem.~\ref{lemma_combine}:
$$
\begin{aligned}
&\left\|\cdot\right\|_{d\times \pi^*}\leq \sqrt{\widetilde{C_d}C_{\pi,\pi^*}}\left\|\cdot\right\|_{d^\pi},\\
&\left\|\cdot\right\|_{d\times \pi_{Q^*,Q_0}}\leq \sqrt{\widetilde{C_d}C_{\pi,\pi^*}}\left\|\cdot\right\|_{d^\pi},
\end{aligned}
$$
where $\widetilde{C_d}=\frac{(N+n)C_d}{N+nC_d}$.

Similar to the derivation of last section, we have
\begin{align*}
    J(\pi^*)-J(\pi)\leq \frac{2\gamma}{1-\gamma}\left(\left\|Q^*-Q_0\right\|_\infty+\left\|Q^\pi-Q_0\right\|_\infty\right)+\frac{\sqrt{\widetilde{C_d}C_{\pi,\pi^*}}+1}{(1-\gamma)^4}\frac{56R_{\max}^2\log\frac{|\gF|^2}{\delta}}{3(n+N)}
\end{align*}

\subsection{Proof of Thm.~\ref{thm_diffusion_error}}\label{appx_53}

% \begin{lemma}[Telescoping Lemma~\cite{kakade}]
%     For any policy $\pi$ and dynamical models $T$, $\widehat{T}$, 
%     $$V^{\pi,\widehat{T}}-V^{\pi,T}=\frac{\gamma}{1-\gamma}\E_{s\sim d, a\sim \pi}[G^{\pi, \widehat{T}}(s,a)],$$
%     where $G^{\pi, \widehat{T}}(s,a)=\E_{s'\sim \widehat{T}(\cdot|s,a)}V^{\pi,\widehat{T}}(s')-\E_{s'\sim T(\cdot|s,a)}V^{\pi,\widehat{T}}$.
% \end{lemma}
Define th transition function learned by the model is $\widetilde T$. To simplify the notation, we use $\widetilde{\gT}^\pi$ to denote the Bellman operator under learned model, i.e.,
$$\widehat{\gT}^\pi f=\argmin_g\sum_{i=1}^N(g(s,a)-r(s,a)-\gamma V^\pi_f(\tilde{s}'))^2,$$
where $\tilde s'\sim \widetilde T(s'|s,a)$. 
\begin{lemma}\label{lemma_diffusion_model_error}
    $\left\|(\gT^\pi)^\infty Q_0 - (\widetilde{\gT}^\pi)^\infty Q_0\right\|_{\hat{d}^\pi}\leq \frac{\gamma}{1-\gamma}L\epsilon^d_m$
\end{lemma}
\begin{proof}
Note that
    $$\left\|\gT^\pi Q_0-\widetilde{\gT}^\pi Q_0\right\|_{\hat{d}^\pi}=\gamma \left\|Q_0(s',\pi)-Q_0(\tilde{s}',\pi)\right\|_{\hat{d}^\pi},$$
where $s'\sim T$, $\tilde{s}'\sim \widetilde{T}$. 

Since $Q_0$ is $L$-Lipschitz, 
\begin{align*}
    \left\|Q_0(s',\pi)-Q_0(\tilde{s}',\pi)\right\|_{\hat{d}^\pi}\leq L\left\|s'-\tilde{s}'\right\|_{\hat{d}^\pi}\leq L\epsilon^d_m
\end{align*}

Then,
\begin{align*}
    &\quad \ \left\|(\gT^\pi)^2 Q_0-(\widetilde{\gT}^\pi)^2 Q_0\right\|_{\hat{d}^\pi}\\
    &=\gamma \left\|(\gT^\pi Q_0)(s',\pi)-(\widetilde{\gT}^\pi Q_0)(\tilde{s}',\pi)\right\|_{\hat{d}^\pi}\\
    &\leq \gamma \left\|(\gT^\pi Q_0)(s',\pi)-(\widetilde{\gT}^\pi Q_0)(s',\pi)\right\|_{\hat{d}^\pi} + \left\|(\widetilde{\gT}^\pi Q_0)(s',\pi)-(\widetilde{\gT}^\pi Q_0)(\tilde{s}',\pi)\right\|_{\hat{d}^\pi}\\
    &\overset{(a)}{=}\gamma \left\|(\gT^\pi Q_0)-(\widetilde{\gT}^\pi Q_0)\right\|_{\hat{d}^\pi} + \left\|(\widetilde{\gT}^\pi Q_0)(s',\pi)-(\widetilde{\gT}^\pi Q_0)(\tilde{s}',\pi)\right\|_{\hat{d}^\pi}\\
    &\leq \gamma^2 L\epsilon^d_m +\gamma L\epsilon^d_m,
\end{align*}
where (a) is because $d^\pi$ is stationary distribution.

Similarly, we have
$$\left\|(\gT^\pi)^\infty Q_0 - \widetilde{\gT}^\infty Q_0\right\|_{\hat{d}^\pi}\leq \sum_{t=1}^\infty \gamma^t L\epsilon^d_m=\frac{\gamma}{1-\gamma}L\epsilon^d_m.$$
\end{proof}

\begin{lemma}\label{lemma_distirbution_error}
    Let $\nu$ and $\mu$ be any state distribution, $f:\gS\to \sR$ be any function with $|f|\leq C$, $\left|\left\|f\right\|_\nu-\left\|f\right\|_\mu\right|\leq C \left\|\nu-\mu\right\|$.
\end{lemma}
\begin{proof}
    \begin{align*}
        \left|\left\|f\right\|_\nu-\left\|f\right\|_\mu\right|&=\left|\sqrt{\sum_s |f(s)|^2\nu(s)}-\sqrt{\sum_s |f(s)|^2\mu(s)}\right|\\
        &\overset{(a)}{=}\frac{\left |\|f\|_\nu^2-\|f\|_\mu^2\right|}{\|f\|_\nu+\|f\|_\mu}\\
        &\overset{(b)}{\leq} \sqrt{\left |\|f\|_\nu^2-\|f\|_\mu^2\right|}\\
        &\leq \sqrt{C^2|\nu(s)-\mu(s)|}\\
        &=C\sqrt{\left\|\nu-\mu\right\|}
    \end{align*}
    (a) is because $a^{1/2}-b^{1/2}=\frac{a-b}{\sqrt{a}+\sqrt{b}}$. (b) is because
    $$\sqrt{\left |\|f\|_\nu^2-\|f\|_\mu^2\right|}\leq \|f\|_\nu+\|f\|_\mu.$$
\end{proof}

Similar to the analysis in the previous section, $\left\|Q^*-\widetilde{Q}\right\|_\nu$ can be decomposed into three terms. 
\begin{equation}\label{eq_decomposition}
    \left\|Q^*-\widetilde{Q}\right\|_\nu
    \leq \gamma \left\|Q^*-Q_0\right\|_{\nu\times \pi_{Q^*,Q_0}}+\left\|Q^{\pi}-(\gT^{\pi})^\infty Q_0\right\|_\nu+\left\|(\gT^{\pi})^\infty Q_0-(\widetilde{\gT}^\pi)^\infty Q_0\right\|_\nu
\end{equation}

According to last section, we only need to consider $\nu=d^\pi$, and obtain that $\left\|Q^{\pi}-(\gT^{\pi})^\infty Q_0\right\|_{d^\pi}=0$.

Therefore, we focus on the last term $\left\|(\gT^{\pi})^\infty Q_0-(\widetilde{\gT}^\pi)^\infty Q_0\right\|_{d^\pi}$.

\begin{align*}
    \left\|(\gT^{\pi})^\infty Q_0-(\widetilde{\gT}^\pi)^\infty Q_0\right\|_{d^\pi}&\overset{(a)}{\leq} V_{\max}\epsilon_d+\left\|(\gT^{\pi})^\infty Q_0-(\widetilde{\gT}^\pi)^\infty Q_0\right\|_{\hat{d}^\pi}\\
    &\overset{(b)}{\leq} V_{\max}\epsilon_d+\frac{\gamma}{1-\gamma}L\epsilon^d_m,
\end{align*}
where (a) uses Lem.~\ref{lemma_distirbution_error}, (b) uses Lem.~\ref{lemma_diffusion_model_error}.

Then we have
\begin{align*}
J(\pi^*)-J(\pi)&\leq \frac{2\gamma}{1-\gamma}\left(\left\|Q^*-Q_0\right\|_\infty+\left\|Q^\pi-Q_0\right\|_\infty\right) + \left(\sqrt{C_{\pi,\pi^*}}+1\right)\left(\frac{1}{1-\gamma}V_{\max}\epsilon_d+\frac{\gamma}{(1-\gamma)^2}L\epsilon^d_m\right)\\
&=\frac{2\gamma}{1-\gamma}\left(\left\|Q^*-Q_0\right\|_\infty+\left\|Q^\pi-Q_0\right\|_\infty\right) + \left(\sqrt{C_{\pi,\pi^*}}+1\right)\left(\frac{1}{(1-\gamma)^2}R_{\max}\epsilon_d+\frac{\gamma}{(1-\gamma)^2}L\epsilon^d_m\right)
\end{align*}

\subsection{Proof of Thm.~\ref{thm_normal_error}}\label{appx_54}

Suppose the dataset generated by the traditional model is $\{(s,a,r,s')\}$. Let the distribution of $s$ in the dataset be $\tilde{d}^\pi$. The next lemma aims to bound the error between $\tilde{d}^\pi$ and the real distribution $d^\pi$:
\begin{lemma}[Lemma B.2 in ~\cite{MBPO}]\label{lemma_model_dist}
    The distance in the state marginal distribution is bounded as:
    $$\left\|\Tilde{d}^\pi-d^\pi\right\|\leq \frac{1}{1-\gamma}\epsilon^t_m.$$
\end{lemma}

% For model-based methods, the decomposition is the same as Eq.~(\ref{eq_decomposition}). Note that $(\gT^\pi)^\infty Q_0=Q^{\pi, T}$, $(\widetilde{\gT}^\pi)^\infty Q_0=Q^{\pi,\widetilde T}$, we can apply Lem.~\ref{lemma_mbpo}:
% $$\left\|(\gT^{\pi})^\infty Q_0-(\widetilde{\gT}^\pi)^\infty Q_0\right\|_{d^\pi}\leq \frac{2\gamma R_{\max}}{(1-\gamma)^2}\epsilon_m.$$

By replacing $\epsilon_d$ with $\frac{1}{1-\gamma}\epsilon^t_m$, we get the final result:

\begin{align*}
J(\pi^*)-J(\pi)&\leq \frac{2\gamma}{1-\gamma}\left(\left\|Q^*-Q_0\right\|_\infty+\left\|Q^\pi-Q_0\right\|_\infty\right) +  \left(\sqrt{C_{\pi,\pi^*}}+1\right)\left(\frac{1}{(1-\gamma)^3}R_{\max}\epsilon_m^t+\frac{\gamma}{(1-\gamma)^2}L\epsilon^t_m\right)
\end{align*}
% \begin{lemma}[Lemma 12 in \cite{fqi}]\label{lemma_coef}
%     Let $\mu$ be any admissible distribution, $\left\|\cdot\right\|_\nu\leq \sqrt{\max \frac{\nu}{\mu}}\left\|\cdot\right\|_\mu$.
% \end{lemma}

% $$\left\|Q^*-Q_{k-1}\right\|_{d^{\pi_k}\times \pi_{Q^*,Q_{k-1}}}\leq C_{k-1}\left\|Q^*-Q_{k-1}\right\|_{d^{\pi_{k-1}}\times \pi_{Q^*,Q_{k-1}}}$$

% Recursively, define $C_{k-1:k-t}=\Pi_{t=k-t}^{k-1}C_t$ and $C_{k-1,k}=1$, we have
% $$\left\|Q^*-Q_k\right\|_{d^{\pi_k}}\leq \sum_{t=0}^k\gamma^{k-t}C_{k-1:k-t}\epsilon_t$$

% \begin{lemma}
%     If $\left\|Q_k-Q^{\pi_k}\right\|$
% \end{lemma}
%%%%%%%%%%%%%%%%%%%%%%%%%%%%%%%%%%%%%%%%%%%%%%%%%%%%%%%%%%%%%%%%%%%%%%%%%%%%%%%
%%%%%%%%%%%%%%%%%%%%%%%%%%%%%%%%%%%%%%%%%%%%%%%%%%%%%%%%%%%%%%%%%%%%%%%%%%%%%%%
\section{Pseudo-Code}\label{appx_code}

\begin{algorithm}[tb!]
   \caption{Energy-Guided Diffusion Sampling in Offline-to-Online Reinforcement Learning}
   \label{alg:example}
\begin{algorithmic}
   \REQUIRE Offline phase loss function  $\{\gL^{Q_\psi}_{\text{offline}}, \gL^{\pi_{\omega}}_{\text{offline}}\}$, online phase loss function $\{\gL^{Q_\psi}_{\text{online}}, \gL^{\pi_\omega}_{\text{online}}\}$, energy-based models $\{\gE_{\phi_1}, \gE_{\phi_2}, \gE_{\phi_3}\}$, noise prediction model $D_{\theta}$.
   \STATE Initialize $\psi, \theta, \phi_1, \phi_2, \phi_3, \omega$, offline replay buffer $\gD_{\text{offline}}$, online replay buffer $\gD_{\text{online}}$, diffusion replay buffer $\gD_{\text{diffusion}}$.
   \WHILE{in \emph{offline training phase}}
   \STATE \% offline policy training using batches from the offline replay buffer $\gD_{\text{offline}}$
  \STATE $\psi \leftarrow \psi - \lambda_Q\nabla_\psi \gL_{\text{offline}}^Q(\psi), \omega \leftarrow \omega - \lambda_\pi \nabla_\omega \gL_{\text{offline}}^{\pi}(\omega)$.
   \ENDWHILE
   \WHILE{in \emph{online training phase}}
   \FOR{each environment step}
   \STATE $\gD_{\text{online}} \leftarrow \gD_{\text{online}} \cup \{(s, a, s', r)\}$
   \ENDFOR
   \IF{step meets $D_{\theta}$ update frequency}
   \STATE Sample data from $\gD_{\text{offline}}\cup\gD_{\text{online}}.$
   \STATE Update $D_\theta$ by minimizing Eq.~(\ref{diffusionloss}).
   \STATE Sample positive samples from $\gD_{\text{online}}$ and generate negative samples with $D_\theta$ based on the positive samples.
   \STATE Update $\gE_{\phi_1}$ by minimizing Eq.~(\ref{eq_energy_1}).
   \STATE Sample positive samples from $\gD_{\text{online}}\cup \gD_{\text{offline}}$ and generate negative samples.
   \STATE Update $\gE_{\phi_2}$ by minimizing Eq.~(\ref{eq_energy_2}).
   \STATE Sample positive samples from $\gD_{\text{online}}\cup \gD_{\text{offline}}$ and generate negative samples.
   \STATE Update $\gE_{\phi_3}$ by minimizing Eq.~(\ref{eq_energy_3}).

   \STATE Score-based sampling with energy guidance and store them as $\gD^{\epsilon}, \gD_{\text{diffusion}} \leftarrow \gD\cup\gD^{\epsilon}$.
   \ENDIF
   \FOR{each gradient step}
   \STATE Sample data from $\gD_{\text{online}}\cup \gD_{\text{diffusion}}$.
   \STATE $\psi \leftarrow \psi - \lambda_Q\nabla_\psi \gL_{\text{online}}^Q(\psi)$
   \STATE $\omega\leftarrow \omega - \lambda_\pi \nabla_\omega \gL_{\text{online}}^{\pi}(\omega)$
   \ENDFOR
   \ENDWHILE
\end{algorithmic}
\end{algorithm}

\section{Experimental Details}
\label{appendix:experiment}
\subsection{Task Description}
\textbf{Adroit Manipulation.} Our empirical evaluation on Adroit manipulation contains 3 domains: pen, door, relocate, where the RL agent is required to solve dexterous manipulation tasks including rotating a pen in specific directions, opening a door, and moving a ball, respectively. The offline datasets are human-v1 datasets in D4RL \cite{d4rl} benchmark, which only contain a few successful non-markovian human demonstrations and thus is pretty difficult for most offline RL approaches to acquire reasonable pre-training performances.

\textbf{AntMaze Navigation.} Our tests on Antmaze navigation benchmark consists of four datasets, namely umaze-v2, medium-diverse-v2, medium-play-v2 and large-play-v2 from the D4RL \cite{d4rl}. The objective is for an ant to learn how to walk and navigate from the starting point to the destination in a maze environment, with only sparse rewards provided. This task poses a challenge for online RL algorithms to explore high-quality data effectively without the support of offline datasets or additional domain knowledge.

\textbf{MuJoCo Locomotion.} MuJoCo locomotion encompasses several standard locomotion tasks commonly utilized in RL research, such as Hopper, Halfcheetah, Walker2d. In each task, the RL agent is tasked with controlling a robot to achieve forward movement. The D4RL \cite{d4rl} benchmark provides four types of datasets with varying quality for each task: random-v2, medium-v2, medium-replay-v2, medium-expert-v2.

\subsection{Details on Comparison with Model-based Methods}\label{appx_visualize}
\begin{wrapfigure}[15]{R}{0.43 \textwidth}
\vspace{-2mm}
\centering
\includegraphics[width=0.37\textwidth]{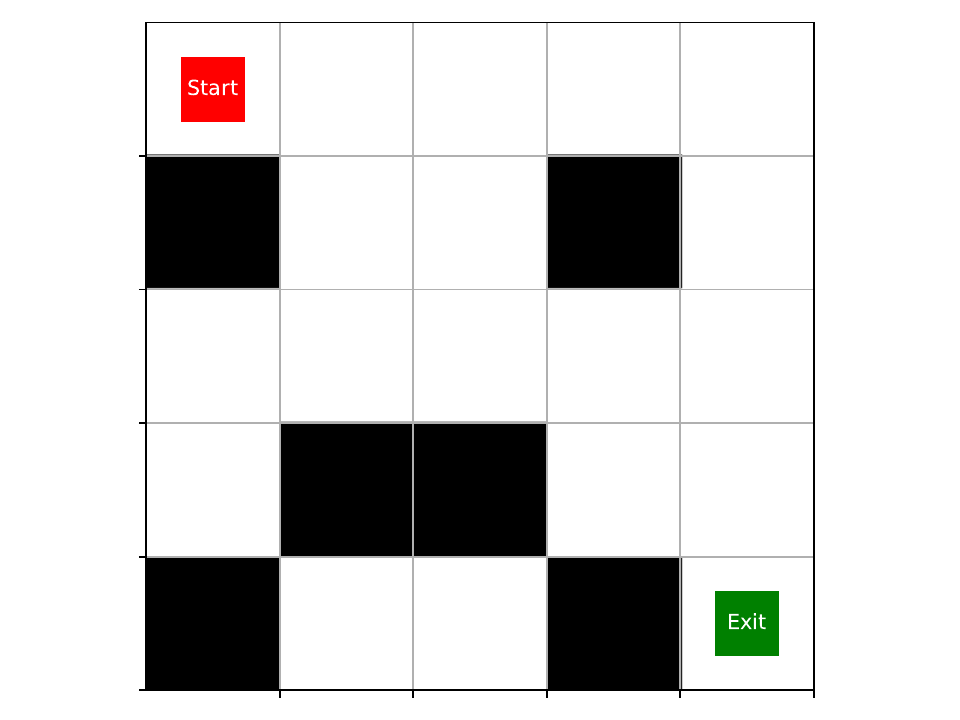}
\caption{Visualization of the Maze MDP.}
\label{fig:maze}
\end{wrapfigure}
% \begin{wrapfigure}[13]{R}{0.43 \textwidth}
%     % \vspace{-5em}
% 	\centering
% 	\makebox[\textwidth][c]{
% 	\includegraphics[width=0.3\paperwidth]{images/maze visualization.pdf}
% 	}
%     \vspace{-1.5em}
%  \caption{Visualization of Maze MDP.}
% 	\label{fig:maze}
%     \vspace{-0.5em}
% \end{wrapfigure}

We conduct the experiment on comparison of distribution differences between the real online distribution and the distribution generated by models a Maze environment, which is visualized in \cref{fig:maze}. With going up, down, left, and right as selectable actions, the agent starts at the upper left corner and the exit is at the lower right corner. The black blocks are occupied and inaccessible. The agent's goal is to reach the exit as quickly as possible, with every step the agent incurs a penalty or, when finally reaching the exit, a reward. 

We visualize the distribution difference in Fig.~\ref{fig:tabularhot}. To quantify this difference, we compute the distribution difference as the deviation in the number of times the agent visited each state, normalized by a factor of 1000. Table~\ref{tab:divergencevalue} explicitly illustrates the divergence values, underscoring that our EDIS effectively generates the intended distribution. In contrast, model-based approaches, particularly those utilizing MLP, struggle to accurately emulate the actual online distribution. While the distribution generated by the transition modeled by diffusion model shows better result, verifying diffusion model has better capability of modeling distributions.
It is noteworthy that the offline dataset distribution exhibits a substantial divergence from the online distribution, rendering direct replay of the offline dataset impractical. Despite the online buffer having a distribution comparable to the real one, its limited dataset size poses a challenge to achieving optimal sample efficiency.

\subsection{Details and Hyperparameters for EDIS}

We use the PyTorch implementation of Cal-QL and IQL from \href{https://github.com/tinkoff-ai/CORL}{https://github.com/tinkoff-ai/CORL}, and primarily followed the author’s recommended parameters~\cite{corl}.
The hyperparameters used in our EDIS module are detailed in the Tab.~\ref{tab:hyperpara}:

\begin{table}[h]
\centering
\caption{Hyperparameters and their values in EDIS}
\begin{tabular}{p{8cm}|p{3cm}}
\toprule
Hyperparameter & Value \\ \midrule
Network Type (Denoising) & Residual MLP \\
Denoising Network Depth & $6$ layers \\
Denoising Steps & $128$ steps \\
Denoising Network Learning Rate & $3 \times 10^{-4}$ \\
Denoising Network Hidden Dimension & $1024$ units \\
Denoising Network Batch Size & $256$ samples \\
Denoising Network Activation Function & ReLU \\
Denoising Network Optimizer & Adam \\
Learning Rate Schedule (Denoising Network) & Cosine Annealing \\
Training Epochs (Denoising Network) & $50,000$ epochs \\
Training Interval Environment Step (Denoising Network) & Every $10,000$ steps \\
\midrule
Energy Network Hidden Dimension & 256 units \\
Negative Samples (Energy Network Training) & $10$  \\
Energy Network Learning Rate & $1 \times 10^{-3}$ \\
Energy Network Activation Function & ReLU \\
Energy Network Optimizer & Adam \\\bottomrule
\end{tabular}
\label{tab:hyperpara}
\end{table}

\begin{figure}[th!]
    \centering
    
    % First Row
    \subfigure[Online (Real) Distribution]{
        \label{fig:online-real-time}
        \includegraphics[width=0.3\textwidth]{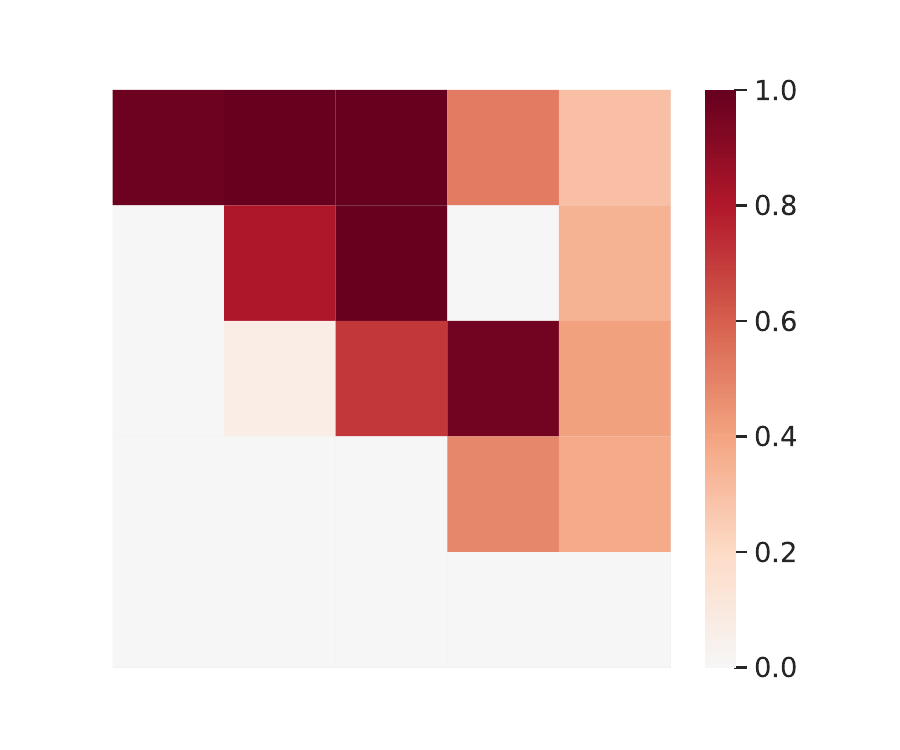}
    }
    \hfill
    \subfigure[MLP Transition Distribution]{
        \label{fig:mlp-conditional}
        \includegraphics[width=0.3\textwidth]{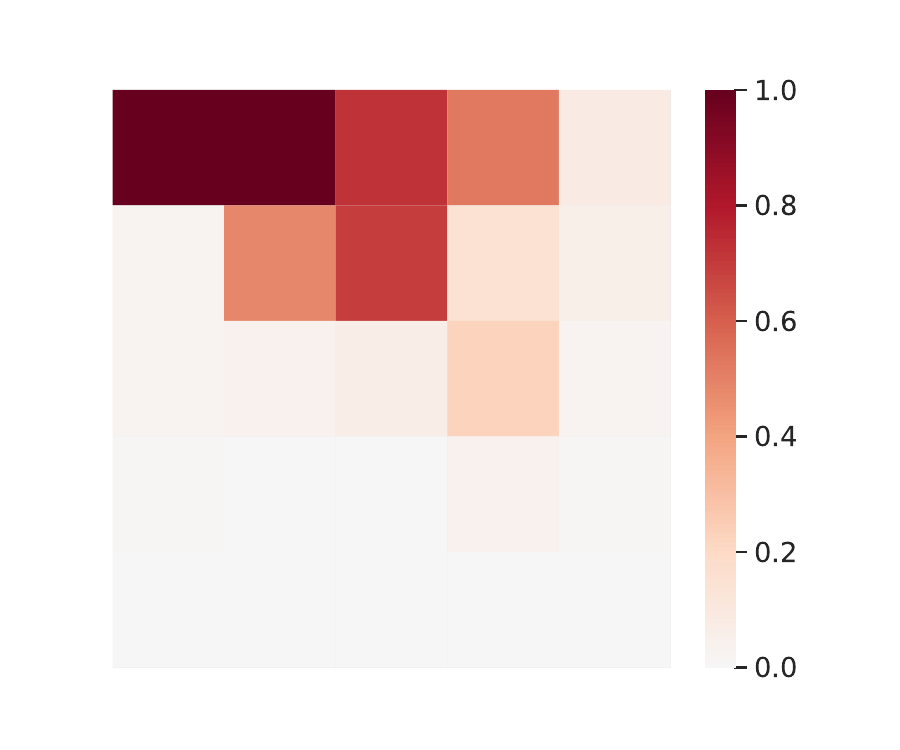}
    }
    \hfill
    \subfigure[Diffusion Distribution]{
        \label{fig:diffusion}
        \includegraphics[width=0.3\textwidth]{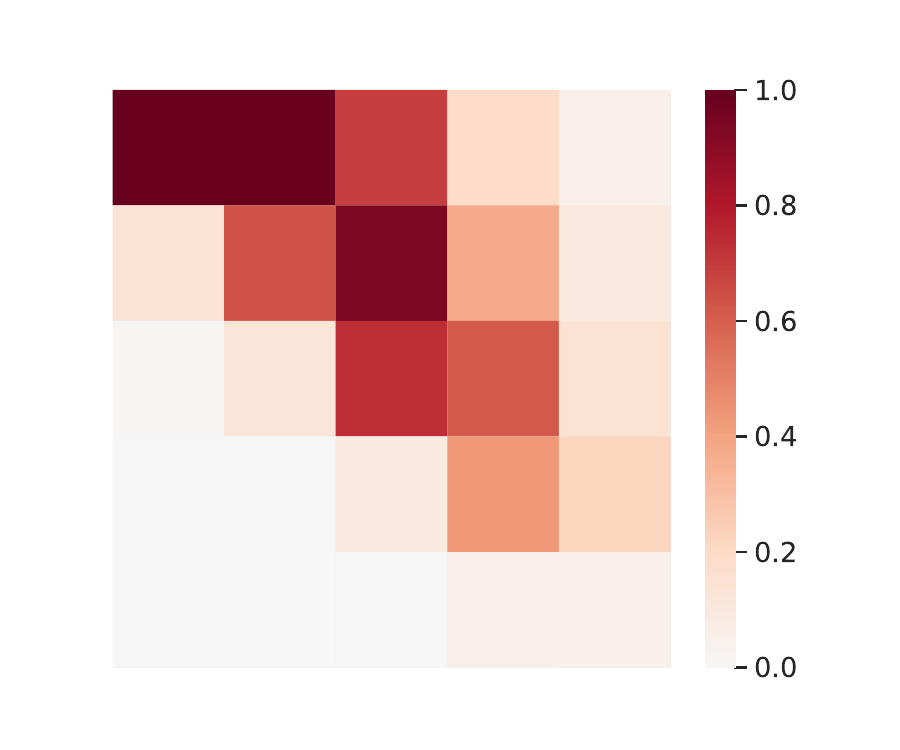}
    }
    \hfill
    \subfigure[Diffusion Transition Distribution]{ 
        \label{fig:diffusion-conditional}
        \includegraphics[width=0.3\textwidth]{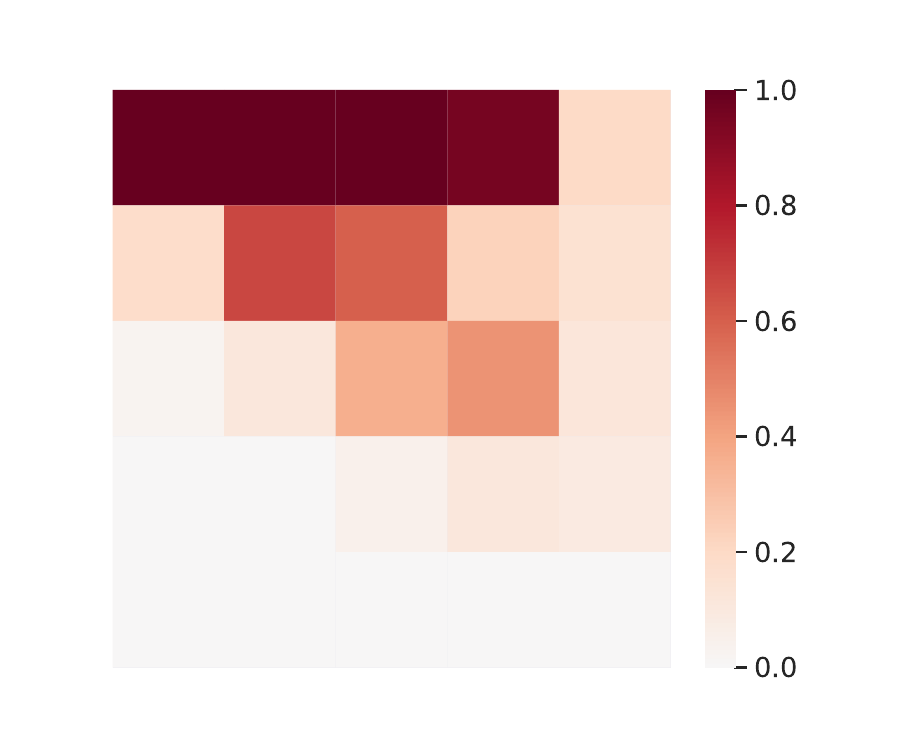}
    }
    \hfill
    \subfigure[Offline Distribution]{ 
        \label{fig:diffusion-offline}
        \includegraphics[width=0.3\textwidth]{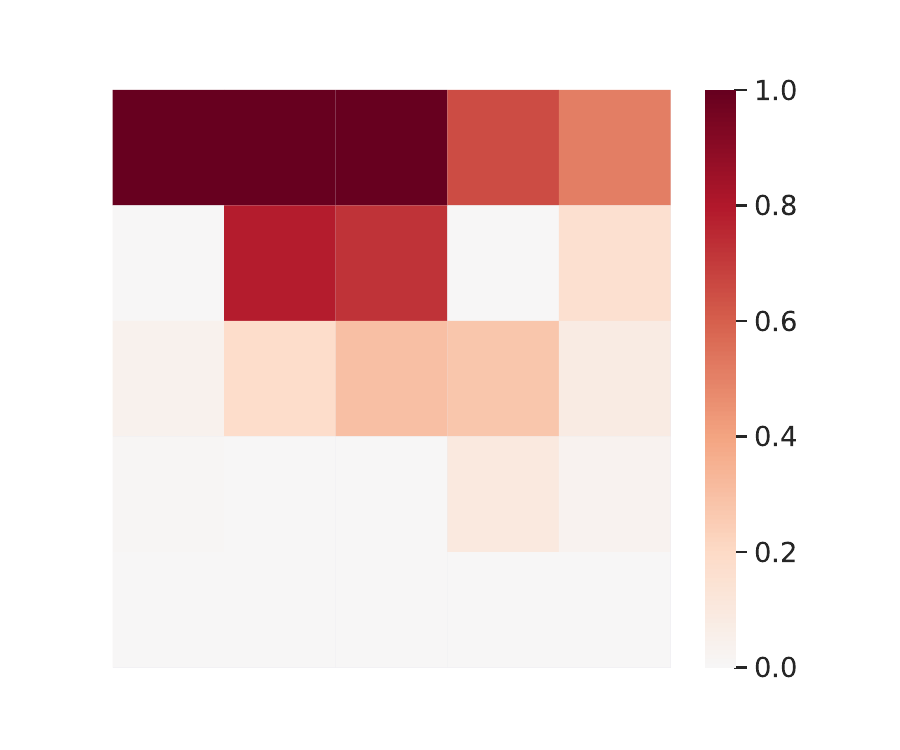}
    }
    \hfill
    \subfigure[Online Buffer Distribution]{ 
        \label{fig:diffusion-buffer}
        \includegraphics[width=0.3\textwidth]{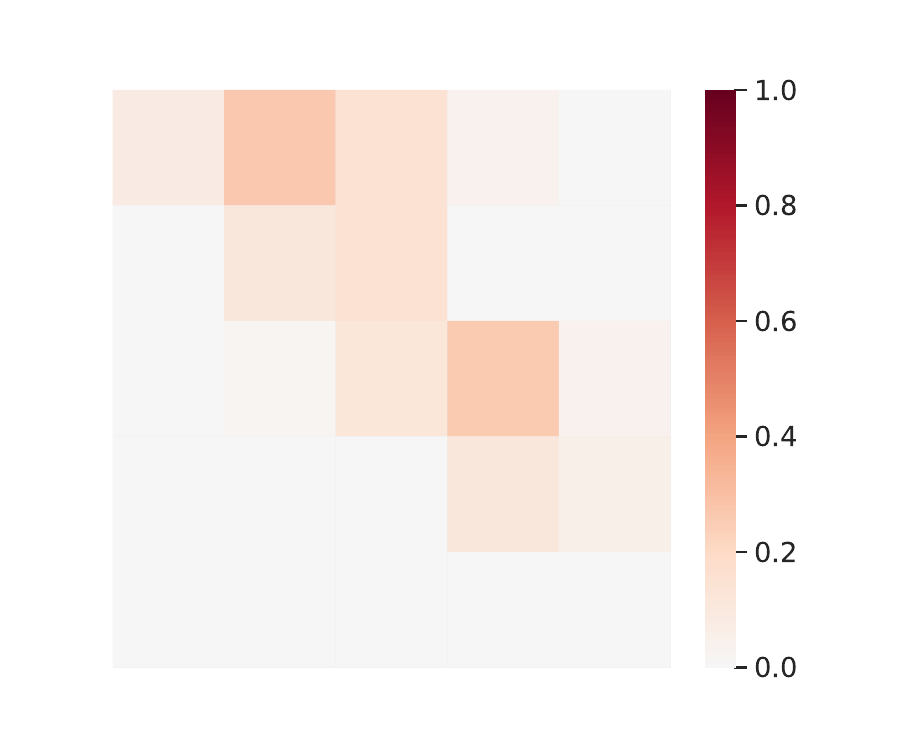}
    }

    \caption{Illustration of the distribution differences between the actual online distribution and the three generated distributions in a Maze environment.}
    \label{fig:tabularhot}
\end{figure}

\begin{table}
  \centering
  \label{tab:divergencevalue}
  \caption{Comparison of distribution divergence between the actual online distribution and generated data distributions in different methods}
  \vspace{2mm}
  \footnotesize
  \begin{tabular}{*{8}{c}}
    \hline
     & Offline & Interaction & MLP Transition & Diffusion Transition  & Diffusion  \\
    \hline
    Divergence & 0.29 & 0.18 &  0.44 & 0.24 & 0.30 \\
    \hline
  \end{tabular}
\end{table}

\subsection{Additional Ablation Study}\label{appx_ablation}

We conducted additional ablation studies on each energy function by omitting their guidance during the reverse-time \emph{SDE} in a total of six environments, including hopper-random-v2, halfcheetah-medium-replay-v2, antmaze-medium-play-v2, antmaze-medium-diverse-v2, and door-human-v1. It can be observed that each energy function plays a functional role and reduces the corresponding divergence value.

The state divergence here is the JS divergence between the two state distributions, which is defined as:
$$D_{\textnormal{JS}}(p,q)=\frac{1}{2}\left(\KL(p,(p+q)/2)+\KL(q,(p+q)/2)\right),$$
where $p$ and $q$ are two distributions and $\KL$ is
$$\KL(p,q)=\int\log(p(x)/q(x))dx.$$

If the two distributions are similar, JS divergence will approach zero, or it will approach one. To calculate it, we apply the techniques of GAN~\cite{gan}. We learn a discriminator by minimizing the following loss function

$$
V(D) = \E_{x\sim p}[\log D(x)]+\E_{z\sim q}[\log(1-D(z)],
$$

and JS divergence can be derived according to the following formula:
$$\max_DV(D)=-\log4 + 2D_{\textnormal{JS}}(p,q).$$

\begin{table*}[t]
\caption{Divergence comparisons for energy function ablation study}
\label{tab:additionalablationenergy}
\vskip 0.15in
\begin{center}
\footnotesize
\begin{tabular}{@{}lcc|cc|cccc@{}}
\toprule
\multirow{2}{*}[-0.5ex]{\textbf{Dataset}} & \multicolumn{2}{c}{State Divergence} & \multicolumn{2}{c}{Action Divergence} & \multicolumn{2}{c}{Transition Divergence} \\
\cmidrule(r){2-3} \cmidrule(l){4-5} \cmidrule(l){6-7}
& w/o energy & w/ energy & w/o energy & w/ energy & w/o energy & w/ energy \\
\midrule
hopper-radnom-v2 & 0.85$\pm$0.02 & \textbf{0.73$\pm$0.03} & 0.51$\pm$0.04 & \textbf{0.39$\pm$0.02} & 0.69$\pm$0.03 & \textbf{0.66$\pm$0.04} \\

halfcheetah-medium-replay-v2 & 0.98$\pm$0.01  & \textbf{0.50$\pm$0.05} & 0.43$\pm$0.15 & \textbf{0.31$\pm$0.03} & 1.01$\pm$0.08 & \textbf{0.88$\pm$0.07} \\

walker2d-medium-expert-v2 & 0.91$\pm$0.02  & \textbf{0.65$\pm$0.12} & 1.62$\pm$0.06 & \textbf{1.32$\pm$0.03} & 0.51$\pm$0.05 & \textbf{0.33$\pm$0.08} \\

antmaze-medium-play-v2 & 0.99$\pm$0.01 & \textbf{0.95$\pm$0.02} & 0.68$\pm$0.09 & \textbf{0.54$\pm$0.05} & 2.05$\pm$0.13 & \textbf{1.85$\pm$0.14} \\ 

antmaze-medium-diverse-v2 & 0.98$\pm$0.00 & \textbf{0.91$\pm$0.02} & 0.38$\pm$0.08 & \textbf{0.27$\pm$0.08} & 0.75$\pm$0.14 & \textbf{0.64$\pm$0.14} \\

door-human-v1 & 0.40$\pm$0.04 & \textbf{0.24$\pm$0.02} & 0.54$\pm$0.07 & \textbf{0.41$\pm$0.08} & 2.57$\pm$0.06 & \textbf{2.52$\pm$0.05} \\ 
\bottomrule
\end{tabular}
\end{center}
\vskip -0.1in
\end{table*}

\begin{figure*}
	\centering
	\makebox[\textwidth][c]{
	\includegraphics[width=0.8\paperwidth]{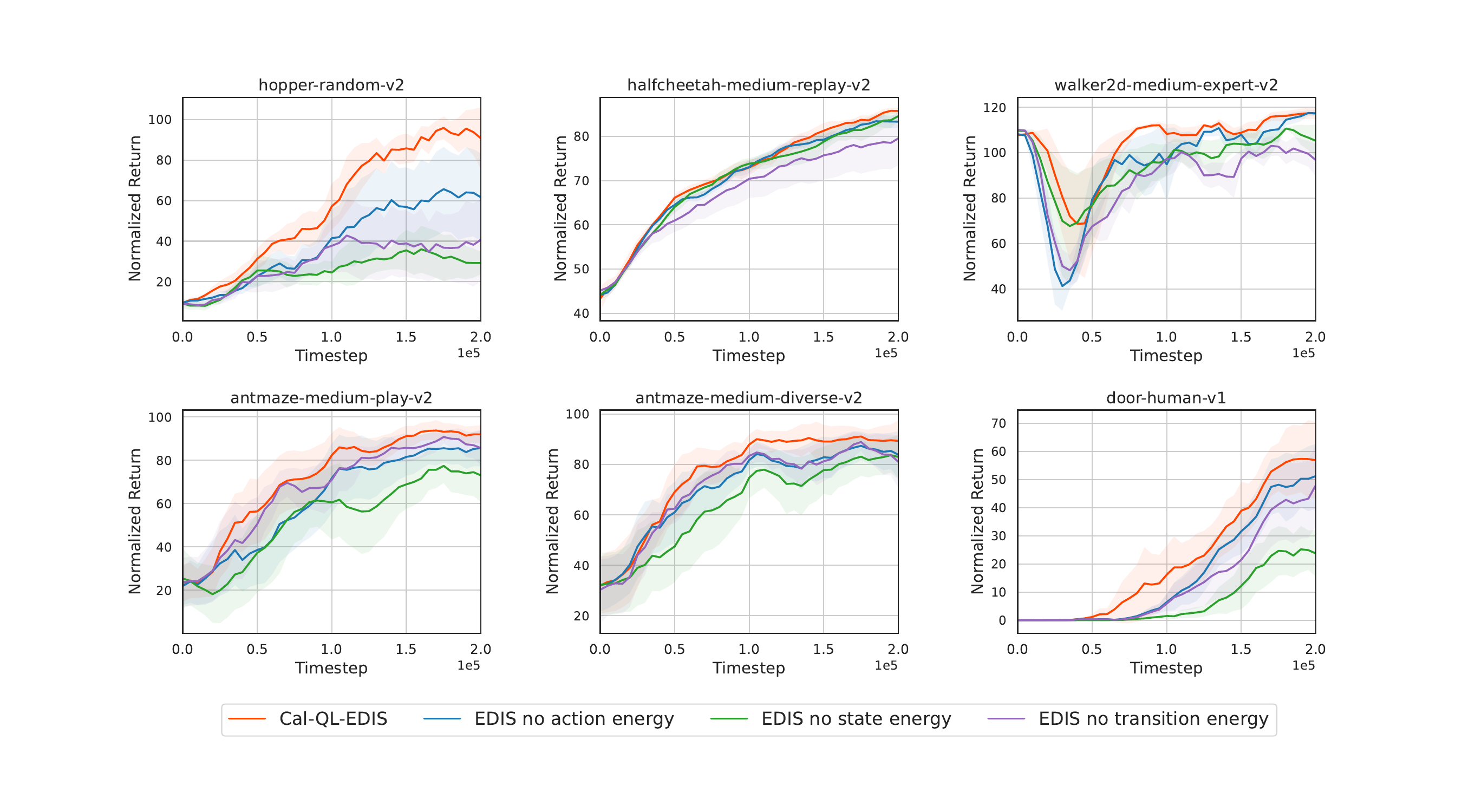}
	}
     \vspace{-2em}
	\caption{Energy Module Abaltion Study of EDIS}
	\label{fig:ablationfull}
\end{figure*}

The aggregated learning curves, summarizing outcomes across all environments from Tab.\ref{tab:performancemain}, are displayed in Fig~\ref{fig:calqlfull} and Fig.~\ref{fig:iqlfull}. They compare the performance of Cal-QL and IQL, both augmented with EDIS, against their respective base algorithms. 
For the evaluation of the normalized score in sparse reward domains, we computed a metric that represents the goal achievement rate for each method. For example, in the adroit environment door-human, we assessed the success rate of opening the door.
\begin{figure*}[htbp]
    % \vspace{-5em}
	\centering
	\makebox[\textwidth][c]{
	\includegraphics[width=0.83\paperwidth]{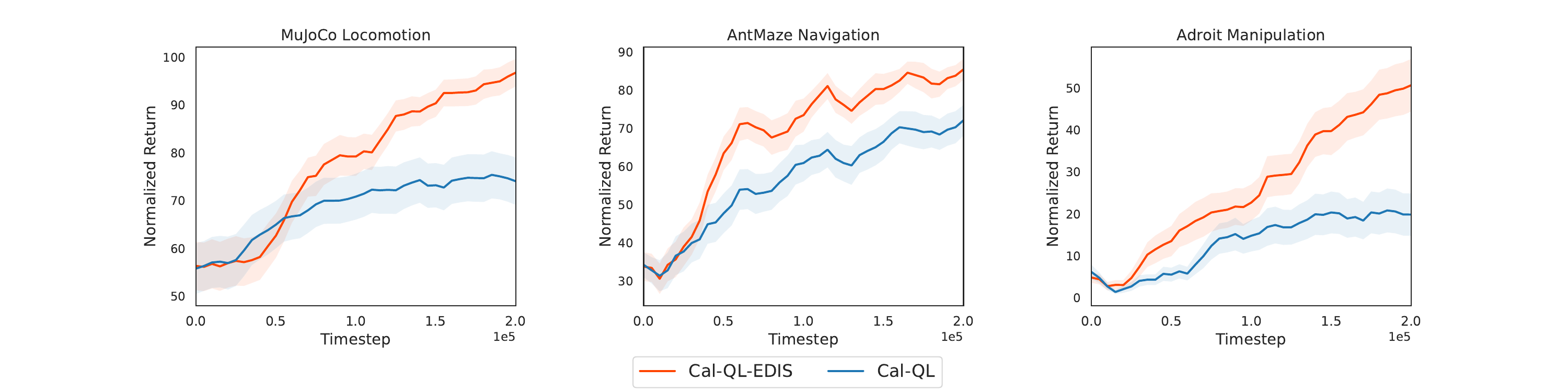}
	}
	\caption{Aggregated learning curves of Cal-QL and Cal-QL-EDIS on MuJoCo Locomotion, AntMaze Navigation and Adroit Manipulation tasks}
	\label{fig:calqlfull}
    \vspace{3mm}
\end{figure*}

\begin{figure*}[htbp]
    % \vspace{-5em}
	\centering
	\makebox[\textwidth][c]{
	\includegraphics[width=0.83\paperwidth]{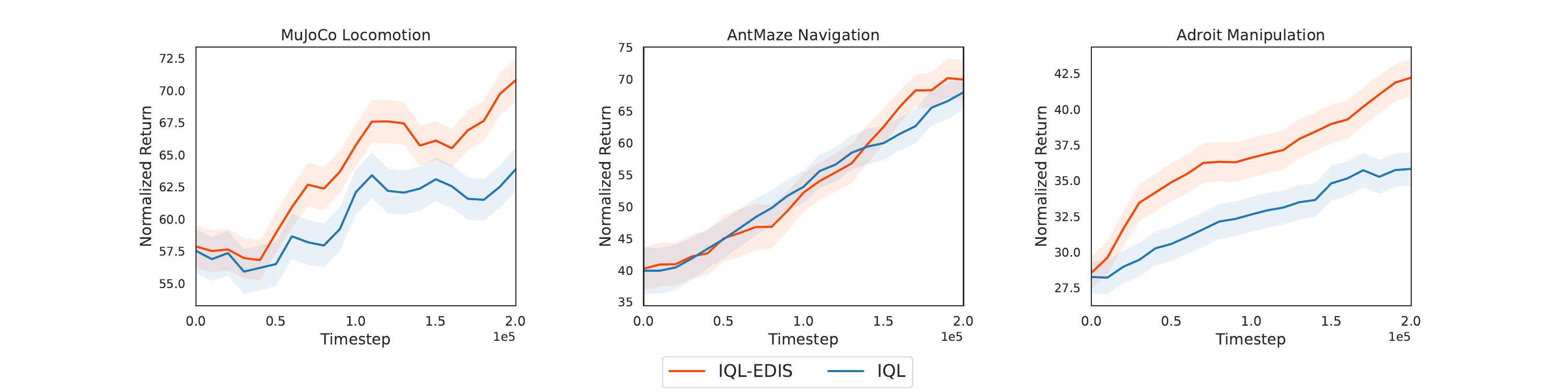}
	}
	\caption{Aggregated learning curves of IQL and IQL-EDIS on MuJoCo Locomotion, AntMaze Navigation and Adroit Manipulation tasks}
	\label{fig:iqlfull}
\end{figure*}

\subsection{Experiments on Visual Environment}
We conduct extra experiments on DMC, which is pixel-based state observations. Our EDIS operates on encoded representations, which are relatively low-dimensional. By integrating EDIS with CQL+SAC, a commonly used offline to online baseline, we observed significant improvements on walker-walk and cheetah-run tasks. In future work, we will explore even more complex environments. This expansion will help further refine our approach and validate its generalization.
\begin{table*}[ht!]
\centering
\small
\vspace{-4mm}
\caption{Experiments on the remaining antmaze environments.}
\label{table:computational_consumption}
% \begin{adjustbox}{max width=\textwidth}
    \begin{tabular}{lcccccc}
    \toprule
    \textbf{Environment} & Cal-QL & Cal-QL-EDIS & IQL & IQL-EDIS \\
    \midrule
    antmaze-umaze-diverse-v2 & 93.4$\pm$4.6 & \textbf{95.9$\pm$2.8} &  51.3$\pm$4.5& \textbf{66.7$\pm$5.0}\\
    antmaze-large-diverse-v2 & 42.3$\pm$2.2  & \textbf{57.1$\pm$ 2.8} & 45.0$\pm$8.7 & \textbf{52.1$\pm$2.6}\\
    \bottomrule
    \end{tabular}
% \end{adjustbox}
\end{table*}

% \subsection{Additional Experiments on Other Base Algorithms}

% In addition to the two base algorithms, IQL~\cite{iql} and CalQL~\cite{calql}, mentioned in the main text, we also attempted to combine EDIS with two other base algorithms, FamO2O~\cite{oncefamily} and PEX~\cite{pex}, the results are shown in Tab.~\ref{tab:additional_base_algorithms}. The results demonstrated the broad applicability of EDIS when integrated with other one-to-one algorithms.

% \begin{table*}[t]
% \caption{Experiments on implementing EDIS on FamO2O and PEX base algorithms.}
% \label{tab:additional_base_algorithms}
% \vskip 0.15in
% \begin{center}
% \footnotesize
% \begin{tabular}{@{}lcc|cccc}
% \toprule
% \multirow{2}{*}[-0.5ex]{\textbf{Dataset}} & \multicolumn{2}{c}{FamO2O} & \multicolumn{2}{c}{PEX}  \\
% \cmidrule(r){2-3} \cmidrule(l){4-5} \cmidrule(l){6-7}
% & Base & EDIS & Base & EDIS  \\
% \midrule
% hopper-medium-v2 & 80.9$\pm$6.7 & \textbf{99.9$\pm$1.4} & 74.2$\pm$10.5 & \textbf{77.0$\pm$14.5} \\
% hopper-medium-replay-v2 & 67.1$\pm$2.7 & \textbf{98.8$\pm$1.2} & 83.0$\pm$18.5 & \textbf{92.2$\pm$7.1}\\
% hopper-medium-expert-v2 & 94.9$\pm$3.8 & \textbf{109.2$\pm$2.0} & 47.3$\pm$28.5 & \textbf{53.4$\pm$15.8}\\
% \bottomrule
% \end{tabular}
% \label{tab:performance}
% \end{center}
% \vskip -0.1in
% \end{table*}

\subsection{Experiments on Sensitivity Analysis}

In our sensitivity analysis, we evaluated the effect of varying the diffusion model's denoising steps and the number of negative samples. We tested denoising steps within the range of 64 to 256 and found comparable performance across this spectrum, indicating that EDIS is robust to changes in this hyperparameter. On the other hand, reducing the number of negative samples from 10 to 5 resulted in a noticeable decline in performance. However, maintaining a minimum of 10 negative samples ensures that performance remains consistent and unaffected.

\begin{figure*}
    \centering
    \includegraphics[scale=0.3]{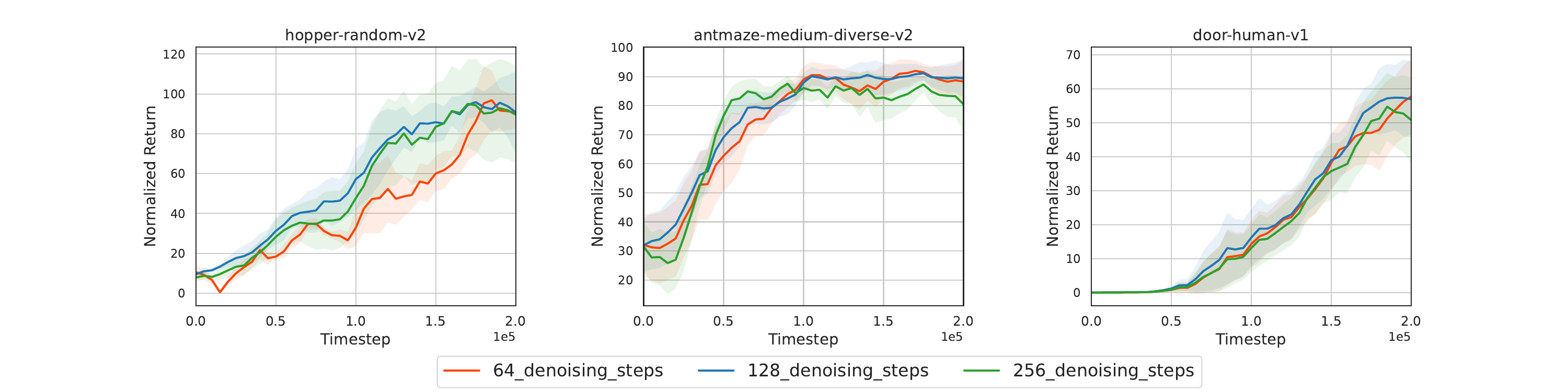}
    \caption{Sensitivity analysis on the training steps for the diffusion model.}\label{fig:analysis_denoising_steps}
\end{figure*}
\begin{figure*}
    \centering
    \includegraphics[scale=0.3]{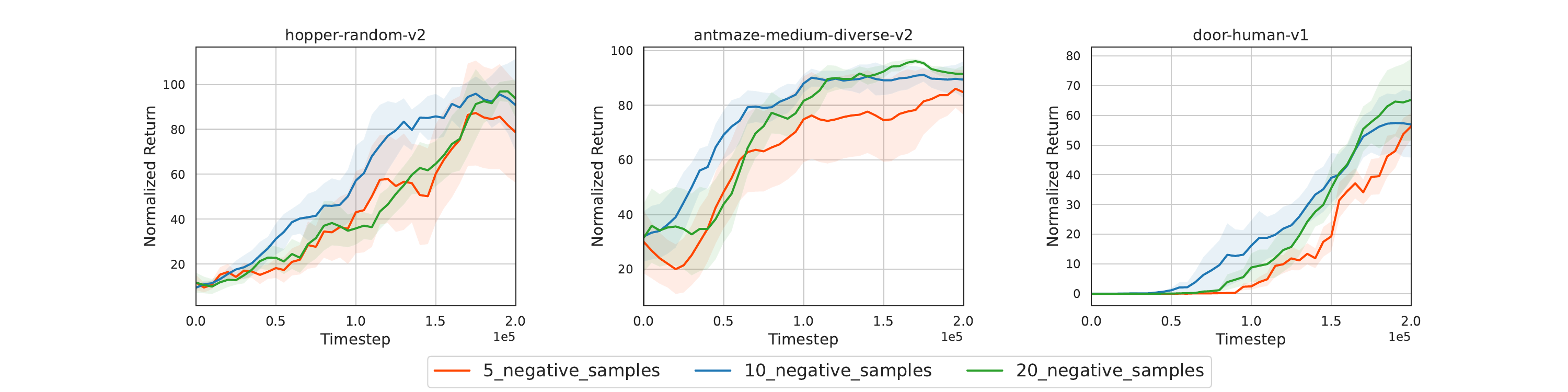}
    \caption{Sensitivity analysis on the number of negative samples.}\label{fig:analysis_negative_number}
\end{figure*}

\subsection{Computational Resources}
We train EDIS integrated with base algorithms on an NVIDIA RTX 4090, with approximately 4 hours required for 0.2M fine-tuning on MuJoCo Locomotion and Adroit Manipulation, while 6 hours for AntMaze Navigation. Specifically, the time cost for each time training our EDIS every time for $50000$ epochs is about $10$ minutes. The detailed computational consumption is shown in Tab.~\ref{table:computational_consumption}. As pointed out in~\cite{elucidate}, the sampling time is faster than prior diffusion designs, which is much shorter compared with training.
The introduction of the diffusion model does indeed entail an inevitable increase in computational and time costs. However, this tradeoff between improved performance and higher computational cost is a common consideration in diffusion model research.  As the field progresses, we anticipate that better solutions will emerge. In our future work, we aim to further refine and optimize the extra costs.

\begin{table*}[ht!]
\centering
\small
\vspace{-4mm}
\caption{Computational consumption of different algorithms.}
\label{table:additional_computational_consumption}
% \begin{adjustbox}{max width=\textwidth}
    \begin{tabular}{lcccc}
    \toprule
    \textbf{Algorithm} & \textbf{Online phase training time} & \textbf{Maximal GPU memory} \\
    \midrule
    Cal-QL (MuJoCo) & 3h  & 2GB \\
    Cal-QL-EDIS (MuJoCo) & 5h & 2GB \\
    Cal-QL (AntMaze) & 3h & 2GB \\
    Cal-QL-EDIS (AntMaze) & 5h & 3GB \\
    Cal-QL (Adroit) & 3h & 2GB \\
    Cal-QL-EDIS (Adroit) & 6h & 3GB \\
    IQL (MuJoCo) & 2h & 2GB \\
    IQL-EDIS (MuJoCo)& 3h & 3GB \\
    IQL (AntMaze) & 3h & 2GB \\
    IQL-EDIS (AntMaze) & 7h & 3GB \\
    IQL (Adroit) & 2h & 2GB \\
    IQL-EDIS (Adroit) & 5h & 3GB \\
    \bottomrule
    \end{tabular}
% \end{adjustbox}
\end{table*}

\section{Additional Related Work}\label{appx_related_work}
\subsection{Offline-to-online Reinforcement Learning.}
Offline-to-online reinforcement learning methods are developed to mitigate the dichotomy between the costly exploration in online RL and the typically suboptimal performance of offline RL.
Previous existing methods do this in a variety of ways: incorporating the offline data into the replay buffer of online RL~\cite{learnfromdemo,leveragingdemo,hybridrl,deepqdemo}, imitating offline policy as auxiliary losses~\cite{rob1, rob2, rob3}, or extracting a high-level skill space for downstream online RL~\cite{relay, opal}. Although these methods improve the sample efficiency of online RL from scratch, they cannot eliminate the need to actively rollout poor policies for data collection~\cite{calql}. 

Another line of work, typically divide the learning process into two phases. Warming up the policy and value functions in the offline phase and use them as initialization in the online phase~\cite{adaptivebc, awac, oncefamily, calql, balanced, iql}. These approaches often employ offline RL methods based on policy constraints or pessimism~\cite{bcq, td3bc, CQL} in the offline phase. 
% 'warm up' the policy and value functions to prime them for the online fine-tuning phase \cite{awac, adaptivebc, oncefamily}. 
However, the conservatism conflicts with the online phase and may induce performance degradation.
% A prevalent issue these methods tackle is the over-conservatism derived from the pre-training phase, which may not adequately consider the distributional shifts that occur in online environments. 
Various strategies have been implemented to tackle this issue,
% circumvent the deterioration of the initial pre-trained policy, 
such as policy expansion \cite{pex}, value function calibration \cite{calql}, Q-ensemble techniques \cite{balanced}, and constraint methods \cite{awac, iql, proto}. 

Despite the advancements, there has been less focus on the crucial aspect of integrating useful data during the fine-tuning phase to enhance training efficiency. Standard practices includes enriching the replay buffer with offline data \cite{calql, pex}, adopting balanced sampling methods for managing both online and offline data sources \cite{balanced, efficientonline}, or building models to performance branch rollout~\cite{moto}. However, directly replaying the offline data causes a distribution shift, and adopting balanced sampling methods introduces large variance~\cite{dualdice} while rolling in the built model suffers from compounding error~\cite{MBPO}. 
% fall short of a comprehensive solution for effectively capitalizing on the wealth of knowledge in offline datasets.
In contrast, our work breaks new ground by proposing a diffusion-based generator specifically designed to generate useful samples by thoroughly leveraging prior knowledge in offline data.

\subsection{Diffusion Model in Reinforcement Learning}
Diffusion models have demonstrated exceptional capabilities in modeling distribution~\citep{SahariaCSLWDGLA22, NicholDRSMMSC22, NicholD21}. Within the reinforcement learning research, Diffuser~\citep{diffuser} uses a diffusion model as a trajectory generator and learns a separate return model to guide the reverse diffusion process toward samples of high-return trajectories.
The consequent work, Decision Diffuser~\citep{DecisionDiffuser} introduces conditional diffusion with reward or constraint guidance for decision-making tasks, further boosting Diffuser’s performance. Diffusion-QL~\cite{DiffusionQL} tracks the gradients of the actions sampled from the behavior diffusion policy to guide generated actions to high Q-value area. SfBC~\cite{sfbc} and Diffusion-QL both employ the technique of resampling actions from multiple behavior action candidates using predicted Q-values as sampling weights.
Expanding the application of diffusion models, SYNTHER~\cite{synther} focuses on leveraging the diffusion model for upsampling data in both online and offline reinforcement learning scenarios. DVF~\cite{conditiondiffusioncontrol} introduces a diffusion model to learn the transition dynamics, which can then be used to estimate the value function. Recently, several concurrent work also investigates generating samples with certain data distribution by diffusion models.  PolyGRAD~\cite{policyguidedtrajdiff} and PGD~\cite{policyguideddiffusion}  embed the policy for classifier-guided trajectory generation, aiming at on-policy world modeling. However, these studies do not focus on the offline-to-online setting, and they model the transition function rather than the distribution directly, which does not eliminate the issue of compounding error.
In the context of offline-to-online reinforcement learning, our research pioneers the utilization of diffusion-based models to actively generate valuable samples. This departure from passive reuse of offline data marks a novel approach, emphasizing the active role of diffusion models in sample generation for enhanced learning outcomes.

\end{document}

%% file: math_command.tex
\usepackage{amsmath,amsfonts,bm}

\def\eqref#1{equation~\ref{#1}}
\def\1{\bm{1}}

\DeclareMathAlphabet{\mathsfit}{\encodingdefault}{\sfdefault}{m}{sl}
\SetMathAlphabet{\mathsfit}{bold}{\encodingdefault}{\sfdefault}{bx}{n}

\def\gA{{\mathcal{A}}}

\def\gD{{\mathcal{D}}}
\def\gE{{\mathcal{E}}}
\def\gF{{\mathcal{F}}}

\def\gL{{\mathcal{L}}}

\def\gN{{\mathcal{N}}}

\def\gS{{\mathcal{S}}}
\def\gT{{\mathcal{T}}}

\def\sR{{\mathbb{R}}}

\def\sV{{\mathbb{V}}}

\newcommand{\E}{\mathbb{E}}

\newcommand{\KL}{D_{\mathrm{KL}}}

\DeclareMathOperator*{\argmax}{arg\,max}
\DeclareMathOperator*{\argmin}{arg\,min}